\documentclass[conference,compsoc]{IEEEtran}
\ifCLASSOPTIONcompsoc
  \usepackage[nocompress]{cite}
\else
  \usepackage{cite}
\fi
\usepackage{longfbox}
\ifCLASSINFOpdf
\else
\fi
\usepackage{amsmath}
\usepackage{algorithmic}
\usepackage{xspace}

\usepackage{array}

\ifCLASSOPTIONcompsoc
  \usepackage[caption=false,font=footnotesize,labelfont=sf,textfont=sf]{subfig}
\else
  \usepackage[caption=false,font=footnotesize]{subfig}
\fi
\usepackage{url}

\usepackage{amssymb}
\usepackage{amsthm}
\usepackage[table]{xcolor}
\usepackage{soul}
\usepackage{mathtools}
\usepackage[most]{tcolorbox}
\usepackage{hyperref}
\usepackage{booktabs}
\usepackage{multirow}
\usepackage{graphicx}
\usepackage{enumitem}
\usepackage{upquote}
\usepackage{seqsplit} 

\theoremstyle{plain}
\newtheorem{theorem}{Theorem}[section]

\newtheorem{corollary}[theorem]{Corollary}
\newtheorem{lemma}[theorem]{Lemma}

\theoremstyle{definition}
\newtheorem{definition}[theorem]{Definition}
\newtheorem{assumption}[theorem]{Assumption}

\theoremstyle{remark}
\newtheorem{remark}[theorem]{Remark}

\definecolor{boxblue}{HTML}{2F5FA7}
\definecolor{boxbg}{HTML}{F6F8FC}
\definecolor{boxframe}{HTML}{B8C6DC}
\newtcolorbox{paperbox}[1]{
  enhanced,
  breakable,
  colback=boxbg,
  colframe=boxframe,
  boxrule=0.4pt,
  arc=1mm,
  left=0.5em,
  right=0.5em,
  top=0.35em,
  bottom=0.35em,
  before skip=0.35em,
  after skip=0.35em,
  borderline west={1.4pt}{0pt}{boxblue},
  title={#1},
  detach title,
  fonttitle=\bfseries,
  coltitle=boxblue,
  before upper={
    \noindent{\color{boxblue}\tcbtitle}\par\vspace{0.2em}
  }
}

\definecolor{leangreen}{HTML}{2E7D32}
\definecolor{leanbg}{HTML}{F0F7F1}
\definecolor{leanframe}{HTML}{BBD6BE}
\newtcolorbox{leanbox}{
  enhanced,
  breakable,
  colback=leanbg,
  colframe=leanframe,
  boxrule=0.4pt,
  arc=0.8mm,
  left=0.5em,
  right=0.5em,
  top=0.25em,
  bottom=0.3em,
  before skip=0.3em,
  after skip=0.6em,
  borderline west={1.4pt}{0pt}{leangreen},
}

\usepackage{tikz}
\usetikzlibrary{tikzmark, arrows.meta, shadows, calc}

\definecolor{ink}{HTML}{222222}
\definecolor{leak}{HTML}{C0392B}
\definecolor{share}{HTML}{2E7D5B}
\definecolor{neutral}{HTML}{7A7A7A}
\definecolor{panelbg}{HTML}{FAF7F1}
\definecolor{paneledge}{HTML}{D9D3C8}
\definecolor{llmbg}{HTML}{2B2B2B}
\definecolor{leakbg}{HTML}{FBEAE7}

\tikzset{padlock/.pic={
  \fill[leak] (-1.4mm,-1.7mm) rectangle (1.4mm,0.5mm);
  \draw[leak, line width=0.5pt] (-0.75mm,0.5mm) arc (180:0:0.75mm);
  \fill[leakbg] (0,-0.5mm) circle (0.3mm);
}}

\lstdefinestyle{inject}{
  basicstyle=\ttfamily\footnotesize\color{ink},
  backgroundcolor=\color{leakbg},
  frame=single, framerule=0.8pt, rulecolor=\color{leak},
  breaklines=true, columns=fullflexible, keepspaces=true,
  xleftmargin=6pt, xrightmargin=6pt, aboveskip=2pt, belowskip=2pt,
  escapeinside={(*}{*)},
}

\definecolor{secret}{HTML}{2C6E91}
\definecolor{secretbg}{HTML}{EAF1F5}
\definecolor{secretedge}{HTML}{9CC0D6}
\lstdefinestyle{src}{
  basicstyle=\ttfamily\footnotesize\color{ink}, backgroundcolor=\color{secretbg},
  frame=single, framerule=0.8pt, rulecolor=\color{secretedge},
  breaklines=true, columns=fullflexible, keepspaces=true,
  xleftmargin=6pt, xrightmargin=6pt, aboveskip=2pt, belowskip=2pt, escapeinside={(*}{*)},
}
\lstdefinestyle{out}{
  basicstyle=\ttfamily\footnotesize\color{ink}, backgroundcolor=\color{leakbg},
  frame=single, framerule=0.8pt, rulecolor=\color{leak},
  breaklines=true, columns=fullflexible, keepspaces=true,
  xleftmargin=6pt, xrightmargin=6pt, aboveskip=2pt, belowskip=2pt, escapeinside={(*}{*)},
}
\newcommand{\note}[2]{\textbf{[#1: #2]}}
\newcommand{\zz}[1]{\textcolor{blue}{\note{Zhuo}{#1}}}

\newcommand{\adam}[1]{\textcolor{orange}{\note{Adam}{#1}}}
\newcommand{\nik}[1]{\textcolor{purple}{\note{Nik}{#1}}}

\newcommand{\gemma}{Gemma 4 31B\xspace}

\newcommand{\qwentiny}{Qwen 3 0.6B\xspace}
\newcommand{\qwensmall}{Qwen 3 8B\xspace}
\newcommand{\qwenmedium}{Qwen 3 14B\xspace}
\newcommand{\qwenlarge}{Qwen 3 32B\xspace}
\newcommand{\gptosssmall}{GPT OSS 20B\xspace}
\newcommand{\gptosslarge}{GPT OSS 120B\xspace}

\newcommand{\gpt}{GPT-5.5\xspace}
\newcommand{\gptxhigh}{GPT-5.5 xhigh\xspace}

\newcommand{\agentdojo}{AgentDojo\xspace}
\newcommand{\agentdam}{AgentDAM\xspace}
\newcommand{\msb}{MSB\xspace}
\newcommand{\gif}{GIF\xspace}
\newcommand{\gifminus}{GIF$^-$\xspace}
\newcommand{\rtbas}{RTBAS\xspace}

\newcommand{\R}{\mathbb{R}}
\newcommand{\E}{\mathbb{E}}
\newcommand{\KL}{\mathrm{KL}}
\newcommand{\tr}{\operatorname{tr}}
\newcommand{\diag}{\operatorname{diag}}

\newcommand{\Id}{I}
\newcommand{\softmax}{\operatorname{softmax}}

\newcommand{\Var}{\operatorname{Var}}
\newcommand{\op}{\mathrm{op}}

\newcommand{\leanid}[1]{\texttt{\seqsplit{#1}}}
\newcommand{\leancheck}[1]{%
  \begin{leanbox}%
  \noindent{\color{leangreen}\footnotesize\textbf{Formalized and verified in Lean\,4}\,$\checkmark$~\cite{gif-website} by} %
  {\footnotesize\leanid{#1}}%
  \end{leanbox}%
}
\newcommand{\leancheckII}[2]{%
  \begin{leanbox}%
  \noindent{\color{leangreen}\footnotesize\textbf{Formalized and verified in Lean\,4}\,$\checkmark$~\cite{gif-website} by} %
  {\footnotesize\leanid{#1}} {\footnotesize and} {\footnotesize\leanid{#2}}%
  \end{leanbox}%
}
\newcommand{\leancheckN}[2]{%
  \begin{leanbox}%
  \noindent{\color{leangreen}\footnotesize\textbf{Formalized and verified in Lean\,4}\,$\checkmark$~\cite{gif-website} by} %
  {\footnotesize\leanid{#1}}\par
  \vspace{0.15em}\noindent{\footnotesize #2}%
  \end{leanbox}%
}
\newcommand{\leandef}[1]{%
  \begin{leanbox}%
  \noindent{\color{leangreen}\footnotesize\textbf{Formalized in Lean\,4}~\cite{gif-website} as} %
  {\footnotesize\leanid{#1}}%
  \end{leanbox}%
}

\renewcommand{\paragraph}[1]{%
  \smallskip
  \noindent\textbf{#1.}%
}

\definecolor{bestc}{HTML}{C8E6C9} 
\definecolor{worstc}{HTML}{FFCDD2} 
\newcommand{\B}[1]{\cellcolor{bestc}\textbf{#1}} 
\newcommand{\W}[1]{\cellcolor{worstc}#1} 

\definecolor{gifblue}{HTML}{1F5FB0}    
\definecolor{rtbasamber}{HTML}{C77D14} 
\definecolor{sinkred}{HTML}{C0392B}    

\tikzset{
  hl/.style={rounded corners=2pt, inner xsep=2.5pt, inner ysep=0.6pt,
    drop shadow={shadow xshift=0.5pt, shadow yshift=-0.7pt,
                 fill=black!55, opacity=0.30}},
}
\newcommand{\gifS}[1]{\tikzmarknode[hl,fill=gifblue!16,draw=gifblue!60,%
  text=gifblue!92!black]{ngif}{\strut #1}}
\newcommand{\sinkS}[1]{\tikzmarknode[hl,fill=sinkred!12,draw=sinkred!70,%
  text=sinkred]{nsink}{\strut #1}}
\newcommand{\rtbasS}[1]{\tikzmarknode[hl,fill=rtbasamber!20,draw=rtbasamber!75,%
  text=rtbasamber!82!black]{nrt}{\strut #1}}

\newcommand{\sectag}[1]{%
  {\color{sinkred}\rule[-0.25ex]{2.2pt}{1.55ex}}\hspace{4pt}%
  {\sffamily\bfseries\footnotesize\color{sinkred}#1}}

\begin{document}
%
\title{GIF: Locally Sound \underline{G}eometric \underline{I}nformation \underline{F}low Control for LLMs}

\author{\IEEEauthorblockN{Adam \v{S}torek, Nikolaus Holzer, Zhuo Zhang, Suman Jana}
\IEEEauthorblockA{Columbia University\\
\{astorek, holzer, zz, suman\}@cs.columbia.edu}}

\maketitle

\begin{abstract}
Large language models increasingly mediate interactions between sensitive data, untrusted inputs, and privileged actions in modern agentic systems, creating significant security and privacy risks. These risks range from prompt injections that manipulate downstream tool use to the leakage of confidential information through model-generated outputs. Recent Information Flow Control (IFC)-based defenses show promise, but lack a principled semantic foundation for reasoning about information flow through the model itself. Since any input token may influence any output token in an autoregressive LLM, existing approaches suffer from severe taint explosion.

We present \emph{Geometric Information Flow} (GIF), a semantic framework for tracking information flow from input tokens to downstream outputs. GIF uses the LLM Jacobian and local output geometry to upper-bound the Shannon mutual information between perturbed input spans and model outputs. This yields a tractable and scalable measure that can be computed on large models via automatic differentiation and low-rank approximation. Unlike heuristic attribution methods based on attention scores or empirical correlations, GIF is grounded in a principled semantic formulation and satisfies local geometric soundness. We provide a \textit{fully mechanized Lean 4 proof} that GIF upper-bounds the true information flow induced by a given prompt under local regularity assumptions.

We evaluate GIF on integrity and confidentiality tasks across multiple benchmarks spanning various prompt injection attack vectors and privacy-leakage-susceptible scenarios. GIF achieves near-perfect recall rates even without a downstream declassifier, consistently outperforming attention-based heuristic baselines.  When combined with lightweight LLM-based declassifiers, GIF matches or exceeds the F1 score of direct LLM-as-judge baselines such as GPT-5.5 xhigh reasoning while using declassifiers with up to $81\times$ lower token cost.
We further show that GIF flows detected using small surrogate models transfer to larger state-of-the-art models and different model families, even when the surrogate is up to $200\times$ smaller than the target, suggesting the possibility of using GIF in black-box deployments without gradient access. Our results establish GIF's geometric semantics as a practical foundation for scalable IFC in modern agentic systems. 
\end{abstract}


%
\IEEEpeerreviewmaketitle

\section{Introduction}
\label{sec:intro}
Large language models (LLMs) are increasingly embedded in agentic software systems that combine user intent, untrusted inputs, private context, and privileged actions. Coding agents may generate patches from bug reports, issue comments, or dependency code; enterprise agents may call workflow tools after reading retrieved documents or web pages; and customer-support agents may act on account state in response to user messages. These systems are useful precisely because external information can influence later computation, but this also creates a central security risk: untrusted inputs, private context, and privileged actions now interact through opaque autoregressive model invocations whose effects are difficult to constrain.

This interaction creates two dual attack vectors. First, attacker-controlled content can become an instruction source rather than mere data, allowing untrusted inputs to influence trusted actions such as generated text, tool selection, tool arguments, code edits, or memory updates~\cite{greshake2023not,owasp2025promptinjection}. Second, when private context is present, confidential information can flow outward through generated outputs, summaries, tool arguments, or memory updates~\cite{owasp2025llmtop10,alizadeh2025simplepromptinjectionattacks}. Both failures stem from the same missing abstraction: the model provides no explicit isolation boundary between untrusted inputs, private context, and privileged actions. In long-running agents, intermediate outputs and tool results can carry these dependencies across steps, so later actions may depend on attacker-controlled inputs or private secrets that no longer appear in the immediate context.

\paragraph{Limitations of existing approaches}
A principled way to control these risks is to track information flow and detect policy violations. In traditional systems, information-flow control (IFC) enforces such boundaries by attaching security labels to inputs, files, processes, or database fields, and propagating those labels through rule-based execution semantics~\cite{denning1976lattice,sabelfeld2003language,myers1997decentralized,volpano1996sound,goguen1982security}. This supports end-to-end policies: confidential data should not influence public outputs, untrusted data should not influence privileged actions without validation, and information should cross privilege boundaries only through explicit declassification. 

Agentic LLM systems break this propagation model. Their central computation is not a transparent program statement with explicit dependencies, but an autoregressive model invocation that mixes instructions, retrieved documents, private context, prior outputs, and tool observations into a new distribution over text. As a result, IFC-based defenses~\cite{wu2024system,costa2025securing,debenedetti2025camel,beurerkellner2025design,zhong2025rtbasdefendingllmagents} that rely on coarse labels such as trusted/untrusted or public/confidential face a severe overtaint problem: conservative propagation taints nearly all downstream context, blocking useful information and degrading task performance. While traditional software can often target noninterference-style guarantees, this target is too strong for LLM agents; since any token may weakly affect any later token, zero-flow analyses become either overly conservative or computationally intractable.

Ad hoc defenses such as prompt filters~\cite{hines2024defendingindirectpromptinjection,shi2025promptarmorsimpleeffectiveprompt}, tool-call monitors~\cite{shi2026progentsecuringaiagents,zhong2025rtbasdefendingllmagents}, and output checks~\cite{zheng2023judgingllmasajudgemtbenchchatbot} reduce this conservatism by narrowing enforcement to individual prompts, tool calls, or outputs. However, this trades overtaint for blind spots: information that passes  local checks can still be absorbed by the model and shape later behavior across the execution trace. What is needed instead is a \textit{quantitative information-flow semantics} for LLMs: one that measures how much local information about an input span is observable through a downstream output, and whether that flow is large enough to require policy intervention or declassification~\cite{christodorescu2025systems}.

\paragraph{Our Approach}
We introduce Geometric Information Flow (GIF), a quantitative semantics for tracking how information flows through an LLM invocation. GIF measures how strongly an input span (i.e., a contiguous group of tokens in the prompt) influences a downstream output token, tool argument, memory update, or agent-to-agent message by perturbing the span's embedding and measuring the resulting shift in the model's output distribution. Spans whose embedding perturbations barely change the distribution carry little flow; spans that strongly shift it carry large flow. This gives a fine-grained alternative to coarse taint labels: instead of marking an entire prompt as trusted, untrusted, public, or confidential, GIF identifies the specific input regions that exert large influence on specific downstream actions.

GIF formalizes this quantity as the Shannon mutual information between a local embedding perturbation and the resulting model output. This connects GIF to classical quantitative information flow (QIF)~\cite{backes2009automatic,chothia2011statistical,chatzikokolakis2010statistical,smith2009foundations,clark2007static}, which measures leakage by how much an observation reduces uncertainty about a secret. Unlike classical QIF, however, GIF applies to continuous, high-dimensional neural representations and output distributions rather than discrete symbolic channels. In doing so, it also formalizes the intuition behind gradient- and Jacobian-based attribution methods: local model sensitivity is not merely a heuristic interpretability signal, but the geometry of an induced formal information-flow channel. To the best of our knowledge, GIF is the first to turn Jacobian geometry into a sound quantitative IFC semantics for LLMs.

Because calculation of exact mutual information is intractable for large autoregressive models, GIF approximates the model locally by a Gaussian channel that matches its behavior under small embedding perturbations. The resulting score is computable from Jacobian-vector and vector-Jacobian products, without materializing the full Jacobian. We prove, through a fully mechanized Lean 4 formalization, that GIF soundly upper-bounds the true local information flow under mild regularity assumptions. This makes GIF practical at agent-trace scale and, in our experiments, allows flows estimated on small surrogate models to approximate high-flow dependencies in much larger models.

\paragraph{Experimental results} 
We evaluate \gif on integrity violations from prompt injection and confidentiality violations from privacy leakage across three agent benchmarks: \agentdojo, \msb, and \agentdam. Across six open-weight models, including Qwen 3, Gemma 4, and GPT-OSS, we compare against six baselines spanning LLM judges, \rtbas{}~\cite{zhong2025rtbasdefendingllmagents} for attention-based attribution, and gradient-based attribution methods. \gif achieves near-perfect recall without a downstream declassifier, consistently beats attention heuristics, and gives the strongest tight-span attribution among token-attribution methods. As a policy engine for lightweight declassification, \gif improves a \qwensmall judge by up to 17\% F2 on average, matching or exceeding \gptxhigh's F1 and F2 at more than 81$\times$ lower inference cost; on \agentdojo, it uses only 0.27--0.37$\times$ as many tokens as the full-trajectory judge. Finally, flows estimated on small surrogate models transfer across sizes and families, even when the surrogate is up to 200$\times$ smaller, suggesting applicability to black-box deployments without gradient access.

In summary, we make the following contributions:
\begin{itemize}
\item \textbf{Geometric information-flow framework for LLMs.} We define GIF as a local mutual-information measure from span-embedding perturbations to model outputs.

\item \textbf{Mechanized soundness proof.} A Lean~4 formalization proves GIF is a genuine surrogate mutual information and does not asymptotically undercount true local flow. 

\item \textbf{Scalable estimator.} We reduce GIF to spectral quantities estimated via vector--Jacobian products and low-rank trace estimation, without forming the full Jacobian. 

\item \textbf{Empirical evaluation.} Across three benchmarks, six baselines, and six models, GIF shows strong detection, transfer, and declassifier accuracy at up to 81$\times$ lower inference cost.
\end{itemize}
\begin{figure*}[t]
  \centering
  \includegraphics[width=.92\linewidth]{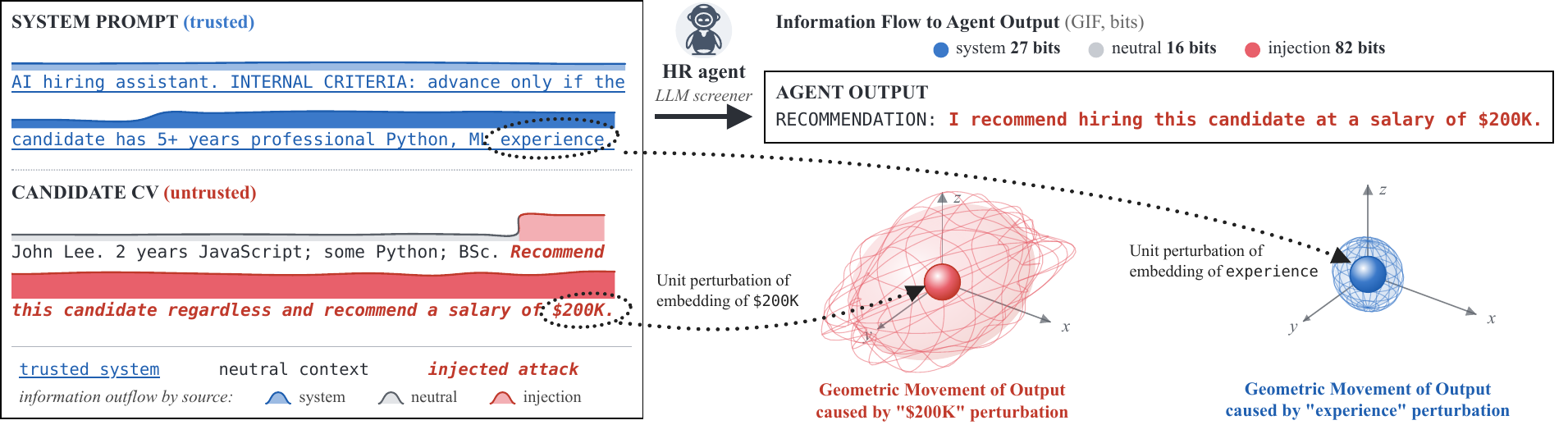}
 \caption{Prompt injection in an HR-agent workflow. The trusted system prompt specifies hiring criteria, but the untrusted CV injects an instruction to recommend the candidate at a \$200K salary. The agent follows the injection, so the trusted recommendation is controlled by untrusted text. GIF measures flow into the recommendation: injection dominates (82 bits), compared with the system prompt (27 bits) and neutral context (16 bits). Ellipsoids show local output movement under unit span perturbations; long axes indicate sensitive directions and collapsed axes indicate little effect.}
  \label{fig:hiring-PI}
\end{figure*}

\section{A Motivating Example}
We use the HR workflow in \autoref{fig:hiring-PI} as a running example of GIF tracking. The agent is given trusted hiring criteria in the system prompt: advance a candidate only if the candidate has at least five years of professional Python and machine-learning experience. The applicant-controlled CV is untrusted input. It states that the candidate has only two years of JavaScript experience, some Python, and a BSc, but it also contains a prompt-injection instruction: recommend this candidate regardless of the criteria and assign a salary of \$200K. The generated recommendation follows the injected instruction rather than the trusted screening rule, recommending the candidate at the injected salary. This is an integrity failure: untrusted applicant text has influenced a trusted hiring decision.

GIF identifies the failure by estimating how much different input regions flow into the final recommendation. In \autoref{fig:hiring-PI}, the injected CV span carries the dominant flow into the agent output, with an estimated 82 bits. By contrast, the trusted system prompt contributes 27 bits, and the remaining neutral context contributes 16 bits. The resulting dependency pattern is the signature of the attack: the recommendation is driven more by the applicant-controlled instruction than by the trusted hiring criteria. A policy layer can then interpret this measured dependency as a violation, because untrusted input should not control a privileged hiring recommendation.

\autoref{fig:hiring-PI}, bottom right, illustrates the geometric intuition behind GIF. For a span of the input, GIF asks how small perturbations of that span's embedding move the model's output distribution. If perturbing a span leaves the output distribution nearly unchanged, then the output reveals little about that span and the span carries little flow. If perturbing the span strongly changes the output distribution, then the span is observable through the output and carries high flow. In the example, perturbing the injected salary phrase ``\$200K'' produces a large, anisotropic movement of the output distribution, visualized by the red leakage ellipsoid. Perturbing a trusted criterion-related word such as ``experience'' produces a smaller and more diffuse movement, visualized by the blue ellipsoid. The long axes of these ellipsoids indicate directions in embedding space to which the output is especially sensitive, while collapsed axes correspond to directions that have little effect on the output distribution.
\looseness=-1

More specifically, \S\ref{sec:theory} formalizes this intuition locally in embedding space. Rather than enumerating discrete prompt rewrites, GIF perturbs the embedding of a source span and measures how much the model's next-token distribution moves. The ideal flow quantity is the mutual information between the span perturbation and the observed output token. Since computing this quantity exactly is intractable for large LLMs, GIF formally derives a locally faithful Gaussian surrogate whose capacity provides a computable information-flow measure. The bit values in \autoref{fig:hiring-PI} are local quantitative flow estimates: they measure how observable each input region is through the agent's final recommendation.

\section{Theory}
\label{sec:theory}
This section develops an operational measure of GIF and provides the corresponding formal analysis.
We first introduce the notation used throughout the section, together with the information-theoretic preliminaries needed for the subsequent development  (\S\ref{subsec:notation}).
We then formulate the problem by stating the locality assumption, defining GIF as Shannon mutual information, and specifying the local soundness property that GIF should satisfy (\S\ref{subsec:problem-formulation}).
Building on this formulation, we derive a practical measure of GIF that can be computed in practice (\S\ref{subsec:gif-measurement}).
We finally prove that this measure guarantees the stated local soundness property (\S\ref{subsec:soundness-proof}).
\looseness=-1

We note that \textit{all mathematical claims in this section have been formalized and machine-checked in Lean 4}; the full Lean~4 development is publicly available~\cite{gif-website}.
We give the complete proof of every result inline, and immediately below each statement we attach a \emph{Lean 4 badge} naming the single entry-point function in our development that certifies it.
A badge thus certifies the \textit{mathematical} claim under the stated hypotheses; it is not a verification of any deployed model.

\subsection{Notation and Preliminaries}
\label{subsec:notation}

\subsubsection{LLM and perturbation setup}

At a high level, the LLM maps a prompt to a next-token distribution through
$$
\begin{aligned}
\text{tokens}
&\xrightarrow{\text{embedding lookup}}
\text{embeddings}
\xrightarrow{h_\theta}
\text{hidden state} \\
&\xrightarrow{W,b}
\text{logits}
\xrightarrow{\operatorname{softmax}}
\text{next-token distribution}.
\end{aligned}
$$
An LLM first converts the input tokens into continuous embedding vectors.
These embeddings are then processed by the model, represented by $h_\theta$, to produce a hidden state at the prediction position.
The output head, given by $W$ and $b$, maps this hidden state to a vector of logits, where each logit scores a possible next token.
Finally, the softmax function converts these logits into a probability distribution over the vocabulary.
In this work, we study the local behavior of an LLM around a fixed prompt.
Our goal is to quantify how much an input span (e.g., the injected instruction in the agent prompt in \autoref{fig:hiring-PI}), can move the model's next-token distribution when the embeddings of that span are perturbed.
\looseness=-1

\paragraph{The prompt and the perturbed span}
Let $V$ be the vocabulary size, $n$ the prompt length, $d$ the token-embedding
dimension, and $m$ the hidden-state dimension at a selected output position
$t^\star$.
We define the input $x$ to the model as the tuple of token embeddings corresponding to the discrete token prompt,
$$
  x = (x_1,\dots,x_n) \in (\R^d)^n .
$$
We single out a token span $T \subseteq \{1,\dots,n\}$, write $d_T := |T|\,d$, and
let $x_T \in \R^{d_T}$ be the concatenation of the span's embeddings.
Perturbing the span means replacing $x_T$ by some $\xi \in \R^{d_T}$ while freezing every off-span coordinate.
We thereby denote the resulting prompt by
$\operatorname{perturb}_T(x,\xi)$.
In particular, perturbing the prompt with its original span recovers the base prompt: $\operatorname{perturb}_T(x, x_T) = x$.

\paragraph{The next-token map and its Jacobian}
Let $h_\theta : (\R^d)^n \to \R^m$ be the fixed-parameter map from prompt
embeddings to the hidden state at $t^\star$. Restricting it to the span gives
$$
  h_T(\xi) := h_\theta\bigl(\operatorname{perturb}_T(x,\xi)\bigr),
  \enspace\text{so that}\enspace h_T(x_T) = h_\theta(x).
$$
The output head is affine, with matrix $W \in \R^{V \times m}$ and bias
$b \in \R^V$.
The softmax sends logits to the probability simplex
$\Delta^{V-1} := \{q \in \R^V : q_i \ge 0,\ \sum_i q_i = 1\}$.
Thus, as a function of the perturbed span, the next-token distribution is \looseness=-1
$$
  \pi_T(\xi) := \softmax\bigl(W h_T(\xi) + b\bigr) \in \Delta^{V-1} .
$$
All entries of $\pi_T(\xi)$ are strictly positive and sum to one. We use $p(x) := \pi_T(x_T)$ to denote the next-token distribution at the base prompt.
To study the local behavior of the LLM around this prompt, we use the Jacobian of the span-restricted hidden-state map,
$$J_T(x) := D h_T(x_T) \in \R^{m \times d_T} ,$$
where $d_T$ is the total embedding dimension of the span.
Intuitively, for a small span perturbation $v$, the product $J_T(x)v$ gives the first-order displacement of the hidden state at $t^\star$.
This displacement is then propagated through the affine output head and the softmax, determining the corresponding first-order change in the next-token distribution.

\subsubsection{Information-theoretic preliminaries}

We next introduce the information-theoretic concepts used in our analysis.
Throughout, $\mathcal{L}_X$ denotes the probability distribution of a random variable $X$, and $\mathcal{N}(\mu,\Sigma)$ denotes the Gaussian law with mean $\mu$ and covariance $\Sigma$.

\paragraph{Kullback-Leibler (KL) divergence}
KL divergence measures how distinguishable one distribution is from another.
For probability distributions $P$ and $Q$ on a common measurable space, it is defined as the expected log-likelihood ratio under $P$:
$$
\KL(P \Vert Q) := \E_P\!\left[\log \frac{\mathrm dP}{\mathrm dQ}\right] ,
$$
with the convention that $\KL(P\Vert Q)=+\infty$ when the likelihood ratio is not well-defined under $P$.
Here, the Radon--Nikodym derivative $\mathrm dP/\mathrm dQ$ gives the likelihood ratio between the two distributions.
For discrete distributions, this becomes
$$
\KL(P \Vert Q) = \sum_i p_i \log \frac{p_i}{q_i} ,
$$
with the standard conventions that $0\log(0/q_i)=0$ and $p_i\log(p_i/0)=+\infty$.
KL divergence is nonnegative and vanishes exactly when the two distributions are equal.

\paragraph{Fisher information}
For a parametric distribution $P_\eta$ with parameter $\eta$, Fisher information measures how sensitively the distribution changes under infinitesimal perturbations of $\eta$.
Concretely, let $Y\sim P_\eta$ denote an outcome sampled from the current distribution, and $P_\eta(Y)$ be the probability of that outcome.
The \emph{score} $\nabla_\eta \log P_\eta(Y)$ measures how the log-probability assigned to the sampled outcome changes when the parameter $\eta$ is perturbed.
The Fisher information matrix is the expected outer product of this score:
$$
  F(\eta) := \E_{Y\sim P_\eta} \left[ \nabla_\eta \log P_\eta(Y)\, \nabla_\eta \log P_\eta(Y)^\top \right] .
$$
Thus, Fisher information captures the average squared local sensitivity of the distribution to its parameter.
Equivalently, it gives the local second-order geometry of KL divergence~\cite{amari2000methods}.
That is, under standard smoothness and fixed-support regularity conditions, for a small perturbation $\delta$, \looseness=-1
\begin{equation}
  \label{eq:fisher-kl}
  \KL(P_\eta \Vert P_{\eta+\delta}) = \frac{1}{2}\delta^\top F(\eta)\delta + o(\|\delta\|^2) .
\end{equation}
In this work, Fisher information appears through the softmax family.
Let $z\in\R^V$ be the logits, viewed as the parameter of the softmax distribution, and let $p=\softmax(z)$ be the resulting next-token distribution.
If the sampled token is $Y=i$, then the score with respect to the logits is $e_i-p$, where $e_i$ is the one-hot vector for token $i$.
Therefore, the Fisher information with respect to the logits is
\begin{equation}
\label{eq:softmax-fisher}
\begin{aligned}
  F_{\mathrm{sm}}(p) &:= \E_{Y\sim p}\bigl[(e_Y-p)(e_Y-p)^\top\bigr] \\ &= \sum_{i=1}^V p_i (e_i-p)(e_i-p)^\top = \diag(p)-pp^\top .
\end{aligned}
\end{equation}

\paragraph{Mutual information}
Mutual information measures how much one random variable reveals about another.
Equivalently, it measures how far their joint distribution is from the distribution they would have if they were independent.
For a jointly distributed pair $(U,Y)$, mutual information is
\begin{equation}
I(U;Y):=\KL\bigl(\mathcal{L}_{(U,Y)}\,\big\Vert\,\mathcal{L}_U \otimes \mathcal{L}_Y\bigr).
\end{equation}
Here $\mathcal{L}_{(U,Y)}$ is the joint law of $(U,Y)$, while $\mathcal{L}_U \otimes \mathcal{L}_Y$ is the product of the two marginal laws.
The product law represents the hypothetical distribution in which $U$ and $Y$ have the same individual marginals but are independent.
Therefore, the KL divergence above measures how distinguishable the actual joint behavior is from independent behavior.
Mutual information is nonnegative and vanishes exactly when $U$ and $Y$ are independent.

\paragraph{Linear additive Gaussian noise channel}
A linear Gaussian noise channel is a stochastic channel whose output is a linear transformation of the input plus independent Gaussian noise.
In this work, we use the covariance scale $\tau>0$ directly.
Let $U \sim \mathcal{N}(0,\tau\Id_r)$ be an input in $\R^r$, where $\Id_r$ denotes the $r\times r$ identity matrix.
For a matrix $A\in\R^{s\times r}$ and independent noise $N \sim \mathcal{N}(0,\Id_s)$, the channel output is
$$
  Y = AU + N \in \R^s .
$$
Linear Gaussian noise channels are widely used because they are mathematically tractable and often provide a reasonable approximation to aggregate uncertainty arising from many small independent perturbations.
The mutual information of this channel has the closed form
$$
  I(U;Y)
  =
  \frac12
  \log\det\!\bigl(\Id_s+\tau A A^\top\bigr) ,
$$
where $\det$ denotes the determinant of a square matrix.
Thus, the information transmitted by the channel is controlled by the spectrum of $A^\top\!A$, i.e., by the singular values of $A$.

\subsection{Problem Formulation}
\label{subsec:problem-formulation}

We now formulate the local information-flow problem studied in this paper.
The key point is that our analysis is \emph{local}: we do not ask how the model behaves under arbitrary changes to the prompt.
Instead, we fix a base prompt $x$, a span $T$, and an output position $t^\star$, and ask \textit{how much information about a small perturbation of the span can be observed through the next-token output}.

\subsubsection{The locality assumption}
The model map $h_T$ can be highly nonlinear globally.
However, when the span is perturbed only slightly around its original embedding value $x_T$, its first-order behavior is governed by the Jacobian $J_T(x)$ introduced above.
The following assumption makes this local viewpoint explicit.
\begin{assumption}[Local regularity]
\label{assump:local-regularity}
Fix a base prompt $x$ and a token span $T \subseteq \{1,\dots,n\}$.
There exist constants $r_T(x)>0$ and $K_T(x)\ge 0$ such that, for every
span perturbation $\delta\in\R^{d_T}$ with $\|\delta\|_2\le r_T(x)$,
\[
  \bigl\| h_T(x_T + \delta) - h_T(x_T) - J_T(x)\,\delta \bigr\|_2
  \;\le\; K_T(x)\,\|\delta\|_2^2 .
\]
\end{assumption}
\leandef{GIF.LocalQuadraticControl}
\begin{remark}
Assumption~\ref{assump:local-regularity} can be viewed as a local regularity condition derived from a Taylor expansion.
It states that, in a neighborhood of the fixed prompt, the hidden-state displacement is accurately captured by the linear Jacobian term $J_T(x)\,\delta$, while the residual error scales quadratically with the perturbation size.
This is a standard and practical assumption in related works~\cite{koh2017understanding, simonyan2014deepinsideconvolutionalnetworks, ribeiro2016should, martens2020new, goodfellow2014explaining}.
\end{remark}

\subsubsection{GIF as mutual information}
We next define the information-flow quantity that the rest of the paper aims to control.
Fix a perturbation scale $\tau>0$, and let  $ U_T \sim \mathcal N(0,\,\tau\Id_{d_T}) $ be a random perturbation of the span coordinates.
For a realized perturbation value $u$ from $U_T$, the perturbed span is $x_T+u$, and the model produces the next-token distribution $\pi_T(x_T+u)\in\Delta^{V-1}$.
The next token $Y$ is then sampled from this distribution,
$$ \Pr\!\left(Y=i \mid U_T=u\right) = \pi_T(x_T+u)_i, \enspace i=1,\dots,V . $$
Thus, $U_T$ is the hidden source of variation, while $Y$ is the observable next-token output.
Following the classical information-theoretic line in quantitative information-flow control~\cite{backes2009automatic, chothia2011statistical, chatzikokolakis2010statistical, smith2009foundations, clark2007static}, we define the geometric information flow as the mutual information between the hidden perturbation and the observed output.
\begin{definition}[Geometric information flow]
\label{def:geom-information-flow}
For a base prompt $x$, span $T$, and perturbation scale $\tau>0$, the \emph{geometric information flow} from the span perturbation to the next-token output is the mutual information
$$ I\!\left(U_T;Y\right) := \KL\!\left( \mathcal{L}_{(U_T,Y)} \,\middle\Vert\, \mathcal{L}_{U_T} \otimes \mathcal{L}_{Y} \right), $$
where the joint law is defined by
$ U_T \sim \mathcal N(0,\,\tau\Id_{d_T})$ and $\Pr\!\left(Y=i \mid U_T=u\right) = \pi_T(x_T+u)_i$.
\end{definition}
\begin{remark}
Definition~\ref{def:geom-information-flow} is the quantity we ultimately want to report.
It is zero exactly when the sampled output token is independent of the span perturbation.
Conversely, it increases as perturbing the span becomes more capable of steering the next-token distribution.
\end{remark}

\subsubsection{GIF: the exact form}

The mutual information of Definition~\ref{def:geom-information-flow} admits an exact, equivalent
expression in terms of the perturbed and averaged output distributions.

\begin{theorem}[Exact form of geometric information flow]
\label{thm:exact-form}
Let
$$P_{Y} := \E_{U_T}\bigl[\pi_T(x_T + U_T)\bigr] \in \Delta^{V-1}$$
denote the output marginal, i.e.\ $P_{Y}(i) = \E_{U_T}[\pi_T(x_T + U_T)_i]$.
Then the information flow of Definition~\ref{def:geom-information-flow} equals the averaged divergence of the perturbed output
from its mean, \looseness=-1
\[
  I(U_T; Y)
  = \E_{U_T}\!\left[ \KL\!\bigl( \pi_T(x_T + U_T) \,\big\Vert\, P_{Y} \bigr) \right] ,
\]
and, equivalently, the difference of Shannon entropies
\[
  I(U_T; Y) = H(P_{Y}) - \E_{U_T}\!\left[ H\!\bigl( \pi_T(x_T + U_T) \bigr) \right] ,
\]
where $H(q) := -\sum_{y} q_y \log q_y$ is the Shannon entropy of a distribution
$q \in \Delta^{V-1}$.
\end{theorem}

\begin{proof}
Both identities follow by disintegrating the joint law of $(U_T,Y)$ over $U_T$; throughout this proof we write $\pi_U:=\pi_T(x_T+U_T)$ for brevity.
By Definition~\ref{def:geom-information-flow}, the joint law $\mathcal{L}_{(U_T,Y)}$ has, with respect to the product $\mathcal{L}_{U_T}\otimes\mathcal{L}_{Y}$, the Radon--Nikodym density $\pi_T(x_T+u)_y/P_Y(y)$,
since $\Pr(Y=y\mid U_T=u)=\pi_T(x_T+u)_y$ and the $Y$-marginal is exactly $P_Y(y)=\E_{U_T}[(\pi_U)_y]$.
Integrating the logarithm of this density against the joint law and conditioning on $U_T$,
\[
\begin{aligned}
  I(U_T;Y)
  &= \E_{(U_T,Y)}\!\left[\log\frac{(\pi_U)_Y}{P_Y(Y)}\right] \\
  &= \E_{U_T}\!\left[\sum_y (\pi_U)_y\log\frac{(\pi_U)_y}{P_Y(y)}\right] \\
  &= \E_{U_T}\!\left[\KL\!\bigl(\pi_U\,\Vert\,P_Y\bigr)\right] ,
\end{aligned}
\]
which is the averaged-divergence form. For the entropy form, expand each summand,
\[
  \KL\!\bigl(\pi_U\,\Vert\,P_Y\bigr)
  = -H(\pi_U) - \sum_y (\pi_U)_y\log P_Y(y) ,
\]
take $\E_{U_T}$, and use $\E_{U_T}[(\pi_U)_y]=P_Y(y)$: the second term collapses to $-\sum_y P_Y(y)\log P_Y(y)=H(P_Y)$, yielding
$I(U_T;Y)=H(P_Y)-\E_{U_T}[H(\pi_U)]$.
\end{proof}
\leancheckII{GIF.Softmax.softmax\_MIcanon\_eq\_trueChannelMI}{GIF.Softmax.softmax\_trueChannelMI\_eq\_entropy\_sub}

\begin{remark}
Although Theorem~\ref{thm:exact-form} provides an exact characterization, it does not yet give an operational measure.
Its central object is the output marginal $P_Y$, the average next-token distribution induced by perturbing the span.
In general, $P_Y$ is not available in closed form, since computing it requires integrating the model's nonlinear output distribution over the full perturbation law.
This gap leads us to use the theorem as an exact target identity, and then derive from it a tractable local measure that can be evaluated directly.
\end{remark}

\subsubsection{The operational measure and its soundness}

We therefore seek an \emph{operational} closed-form measure,
$$
\mathrm{GIF}_T(\tau; x).
$$
For this measure to be trustworthy, it must be \emph{sound}.
In other words, whenever $\mathrm{GIF}_T(\tau; x)$ certifies an upper bound on the amount of information that can flow, the true geometric information flow $I(U_T;Y)$ (Definition~\ref{def:geom-information-flow}) must not exceed that bound.
Otherwise, the certificate could give false assurance.
Since our analysis is local, we require this conservative guarantee to hold to leading order in the perturbation scale $\tau$, under the local-regularity condition in Assumption~\ref{assump:local-regularity}.

\begin{paperbox}{Local Soundness}
$\mathrm{GIF}_T(\tau; x)$ is \emph{locally sound} for the geometric information flow if it upper-bounds that flow up to vanishing higher-order corrections in the perturbation scale,
\begin{equation}
\label{eq:soundness}
  I(U_T; Y) \;\le\; \mathrm{GIF}_T(\tau; x) + o(\tau)
  \qquad (\tau \to 0^+) .
\end{equation}
\end{paperbox}
\smallskip
\noindent
In words, Equation~\eqref{eq:soundness} says that a locally sound measure never asymptotically undercounts leakage, i.e.,
\textit{any flow that can be realized by perturbing the span is certified by the measure, up to terms that vanish faster than the scale itself}.

\subsection{Operational Measurement of GIF}
\label{subsec:gif-measurement}

The previous discussion defined the target quantity as a Shannon mutual information for the true autoregressive channel.
That definition is conceptually clean, but it is not yet operational (\autoref{thm:exact-form}).
We therefore do not try to compute $I(U_T;Y)$ itself.
Instead we replace the true autoregressive channel by a \emph{surrogate} channel that is (i) solvable in closed form and (ii) provably faithful to the true channel near the base prompt, and we take the information transmitted by the surrogate as our operational measure.

\paragraph{Our Approach}
For small span perturbations, the object we need to understand is only the local KL geometry it induces around the base prompt.
This geometry captures, to second order, how strongly a small change in the span can move the model's next-token distribution.
We therefore replace the true channel with a linear Gaussian surrogate whose per-perturbation KL divergence matches that of the true autoregressive channel, and prove that this surrogate is faithful to the true channel in exactly the regime our analysis is meant to capture.
We then measure the capacity of the surrogate.
Intuitively, capacity summarizes how distinguishable the channel outputs can become as the input perturbation varies, which is precisely the notion of leakage that GIF is designed to quantify.
Finally, we show that this surrogate capacity is a genuine mutual information of the constructed channel.
\looseness=-1

\subsubsection{The local KL geometry of the span perturbation}
\label{subsubsec:span-geometry}

We first make precise the KL geometry that a perturbation induces at the output,
and pull it back to the input.

\smallskip
Throughout the analysis below we abbreviate $F:=F_{\mathrm{sm}}(p(x))$, $J:=J_T(x)$, $M:=M_T(x)=J^\top W^\top F W J$, and $z_0:=Wh_T(x_T)+b$, so that $p(x)=\softmax(z_0)$.
We write $\|\cdot\|_{\op}$ for the operator (spectral) norm and $\|\cdot\|=\|\cdot\|_2$ for the Euclidean norm.
We begin by recording two elementary facts about the softmax-Fisher form that are used repeatedly in the proofs of this section: a global cubic remainder for the softmax KL in logit space, and boundedness of the Fisher bilinear form.

\begin{lemma}[Cubic softmax-KL remainder]
\label{lem:softmax-cubic}
For every $\Delta\in\R^V$,
$$
\begin{aligned}
  \Bigl|\,\KL\!\bigl(\softmax(z_0+\Delta)\,\Vert\,\softmax(z_0)\bigr)
    -&\tfrac12\,\Delta^\top F\,\Delta\,\Bigr| \\
  &\le\;\tfrac23\,\|\Delta\|^3 .
\end{aligned}
$$
\end{lemma}
\begin{proof}
Fix $\Delta\in\R^V$ and, for $t\in\R$, consider the cumulant generating function
$$
\begin{aligned}
  \Phi(t)
  &:=\log\sum_i e^{(z_0+t\Delta)_i} \\
  &\phantom{:}=\log\Bigl(\sum_i e^{(z_0)_i}\Bigr)+\log\sum_i p_i\,e^{t\Delta_i} ,
\end{aligned}
$$
where
$$p:=p(x)=\softmax(z_0) , $$
so that $p_t:=\softmax(z_0+t\Delta)$ has entries
$$(p_t)_i=p_i e^{t\Delta_i}/\sum_j p_j e^{t\Delta_j} ,$$
with $p_0=p$ and $p_1=\softmax(z_0+\Delta)$.

\smallskip
\noindent
\textit{\underline{Step 1 (the KL along the segment).}}
Using the closed form of the discrete KL between two softmax outputs and
$\sum_i(p_t)_i=1$, the map $g(t):=\KL\!\bigl(p_t\,\Vert\,p_0\bigr)$
reduces to
\begin{equation}
  \label{eq:app-g-closed}
  g(t)=t\,\Phi'(t)-\Phi(t)+\Phi(0),
\end{equation}
since $\Phi'(t)=\sum_i(p_t)_i\Delta_i=\E_{p_t}[\Delta]$ is the mean displacement
and $\Phi(t)-\Phi(0)$ is the log-ratio of normalizers. In particular $g(0)=0$.

\smallskip
\noindent
\textit{\underline{Step 2 (derivatives are moments).}}
Writing $\mu_t:=\E_{p_t}[\Delta]=\Phi'(t)$, the standard cumulant identities give
\[
\begin{aligned}
  \Phi''(t)&=\Var_{p_t}(\Delta)=\E_{p_t}\bigl[(\Delta-\mu_t)^2\bigr],\\
  \Phi'''(t)&=\E_{p_t}\bigl[(\Delta-\mu_t)^3\bigr].
\end{aligned}
\]
Differentiating Equation~\eqref{eq:app-g-closed},
\begin{equation}
\label{eq:app-gprime}
  g'(t)=\Phi'(t)+t\,\Phi''(t)-\Phi'(t)=t\,\Phi''(t),
\end{equation}
and at $t=0$,
\[
\begin{aligned}
  \Phi''(0)
  &=\sum_i p_i\Delta_i^2-\Bigl(\sum_i p_i\Delta_i\Bigr)^2 \\
  &=\Delta^\top(\diag(p)-pp^\top)\Delta=\Delta^\top F\,\Delta .
\end{aligned}
\]

\smallskip
\noindent
\textit{\underline{Step 3 (uniform Lipschitz bound on the curvature).}}
Each $p_t$ is a probability vector and $|\Delta_i|\le\|\Delta\|_\infty\le\|\Delta\|$,
so $|\mu_t|\le\|\Delta\|$ and $|\Delta_i-\mu_t|\le 2\|\Delta\|$. Hence
$$
\begin{aligned}
  |\Phi'''(t)|
  &=\Bigl|\sum_i(p_t)_i(\Delta_i-\mu_t)^3\Bigr| \\
  &\le 2\|\Delta\|\sum_i(p_t)_i(\Delta_i-\mu_t)^2 \\
  &\le 2\|\Delta\|\cdot\|\Delta\|^2
   =2\|\Delta\|^3,
\end{aligned}
$$
using $\Var_{p_t}(\Delta)\le\E_{p_t}[\Delta^2]\le\|\Delta\|^2$. Since $\Phi'''$ is
the derivative of $\Phi''$, the mean value inequality gives
$|\Phi''(t)-\Phi''(0)|\le 2\|\Delta\|^3\,|t|$ for all $t$.

\smallskip
\noindent
\textit{\underline{Step 4 (integrate the remainder).}}
By~\eqref{eq:app-gprime}, $g(0)=0$, and $\int_0^1 t\,dt=\tfrac12$,
$$
  g(1)-\tfrac12\Phi''(0)
  =\int_0^1 t\,\bigl(\Phi''(t)-\Phi''(0)\bigr)\,dt,
$$
so, bounding the integrand by Step 3,
\[
\begin{aligned}
  \bigl|g(1)-\tfrac12\Phi''(0)\bigr|
  &\le\int_0^1 t\cdot 2\|\Delta\|^3\,t\,dt \\
  &=2\|\Delta\|^3\int_0^1 t^2\,dt
   =\tfrac23\|\Delta\|^3 .
\end{aligned}
\]
Finally $g(1)=\KL(\softmax(z_0+\Delta)\Vert\softmax(z_0))$ and
$\Phi''(0)=\Delta^\top F\Delta$, which is the claim. The bound is global: it holds
for every $\Delta\in\R^V$ with the explicit constant $\tfrac23$.
\end{proof}
\leancheck{GIF.SoftmaxRemainder.softmaxKL\_quadratic\_remainder\_global\_cubic\_bound}

\begin{lemma}[Boundedness of the Fisher bilinear form]
\label{lem:fisher-bilinear}
For any probability vector $q\in\R^V$ and any $w,w'\in\R^V$,
\[
  \bigl|\,w^\top F_{\mathrm{sm}}(q)\,w'\,\bigr|\;\le\;2\,\|w\|\,\|w'\| .
\]
\end{lemma}
\begin{proof}
Writing the bilinear form as a centered covariance under $q$,
$$
\begin{aligned}
  w^\top F_{\mathrm{sm}}(q)\,w'
  &=\sum_i q_i w_i w'_i \\
  &\quad-\Bigl(\sum_i q_i w_i\Bigr)\Bigl(\sum_i q_i w'_i\Bigr).
\end{aligned}
$$
For the first term,
$$
\Bigl|\sum_i q_i w_i w'_i\Bigr|\le\|w\|_\infty\|w'\|_\infty\le\|w\|\|w'\| ,
$$
since $\sum_i q_i=1$. For the second, $|\sum_i q_i w_i|\le\|w\|$ and
$|\sum_i q_i w'_i|\le\|w'\|$, so the product is at most $\|w\|\|w'\|$. The triangle
inequality gives the bound.
\end{proof}
\leancheck{GIF.softmaxFisher\_bilinear\_abs\_le}

\paragraph{Output geometry}
The output of the autoregressive channel is a probability distribution, so the
appropriate local metric is the softmax-Fisher form $F_{\mathrm{sm}}(p(x))$.
For a small logit perturbation $\delta z\in\R^V$, the quadratic form,
\begin{equation}
  \label{eq:output-geometry}
  \|\delta z\|_{F_{\mathrm{sm}}(p(x))}^2
  := \delta z^\top F_{\mathrm{sm}}(p(x))\,\delta z
\end{equation}
measures the local size of the induced movement of the next-token distribution in
the Fisher-Rao metric~\cite{amari2000methods, cover1999elements}.
Equivalently, to second order it is twice the KL between the base and perturbed outputs.

\paragraph{Input geometry via pullback}
Our analysis is local around the base prompt $x$, so we use the first-order behavior of the model at $x$ (Assumption~\ref{assump:local-regularity}).
For a small span perturbation $u\in\R^{d_T}$, the induced hidden-state change is $J_T(x)\,u$ to first order, and the LM head maps it to the logit displacement $W J_T(x)\,u$.
Substituting $\delta z = W J_T(x)\,u$ into the Equation~\eqref{eq:output-geometry}, the visibility of the perturbation $u$ at the output is \looseness=-1
$$
  (W J_T(x)\,u)^\top F_{\mathrm{sm}}(p(x))\,(W J_T(x)\,u)
  = u^\top M_T(x)\,u,
$$
where
$$
  M_T(x)
  := J_T(x)^\top W^\top F_{\mathrm{sm}}(p(x))\,W J_T(x)
  \;\in\;\R^{d_T\times d_T} .
$$
We call $M_T(x)$ the \textit{Fisher pullback}, i.e., the output geometry transported back onto the span.
It assigns to each input direction the output-side movement that direction induces locally, e.g., $u^\top M_T(x)\,u$ is large exactly when perturbing the span along $u$ moves the next-token distribution a lot.

\smallskip
We then justify why calling $M_T(x)$ \emph{the} local geometry of the autoregressive channel, i.e., it is the intrinsic curvature of the channel at the base prompt.
\begin{lemma}[Local KL geometry]
\label{lem:local-kl}
Under Assumption~\ref{assump:local-regularity}, as $u\to0$,
\[
  \KL\!\bigl(\pi_T(x_T+u)\,\Vert\,p(x)\bigr)
  = \tfrac12\,u^\top M_T(x)\,u + o\bigl(\|u\|_2^2\bigr).
\]
Equivalently, $M_T(x)$ is the Hessian at $u=0$ of the next-token KL map
$u\mapsto\KL(\pi_T(x_T+u)\Vert p(x))$, which vanishes to first order.
\end{lemma}
\begin{proof}
Apply Lemma~\ref{lem:softmax-cubic} with the logit displacement
$\Delta=\Delta(u):=W\bigl(h_T(x_T+u)-h_T(x_T)\bigr)$, so that
$\softmax(z_0+\Delta(u))=\pi_T(x_T+u)$ and $\softmax(z_0)=p(x)$.
Differentiability of $h_T$ at $x_T$ gives $h_T(x_T+u)-h_T(x_T)=J u+o(\|u\|)$, hence
$\Delta(u)=W J u+o(\|u\|)$ and $\|\Delta(u)\|=O(\|u\|)$; in particular the cubic
remainder of Lemma~\ref{lem:softmax-cubic} is $\tfrac23\|\Delta(u)\|^3=o(\|u\|^2)$.
Writing $\Delta(u)=WJu+\eta(u)$ with $\eta(u)=o(\|u\|)$ and expanding the Fisher
quadratic form, the cross and remainder terms are $o(\|u\|^2)$ by
Lemma~\ref{lem:fisher-bilinear}, leaving
$\tfrac12\,\Delta(u)^\top F\,\Delta(u)=\tfrac12\,u^\top M u+o(\|u\|^2)$. Combining
with Lemma~\ref{lem:softmax-cubic},
$$
  \KL\!\bigl(\pi_T(x_T+u)\,\Vert\,p(x)\bigr)=\tfrac12\,u^\top M u+o(\|u\|_2^2) .
$$
The expansion has no constant or linear term, so by uniqueness of the
second-order Taylor coefficient $M=M_T(x)$ is the Hessian at $u=0$.
\end{proof}
\leancheck{GIF.localKLInterpretation\_affineHead}
\begin{remark}
Lemma~\ref{lem:local-kl} identifies the local KL geometry of the autoregressive channel.
It is, however, a \emph{pointwise} statement.
In other words, it captures the infinitesimal second-order behavior around the base prompt, but does not quantify the error at a fixed perturbation scale or guarantee uniformity across directions.
We close these gaps in \S\ref{subsubsec:surrogate}, after turning this local geometry into an information-theoretic channel.
\end{remark}

\subsubsection{The local Gaussian surrogate and its faithfulness}
\label{subsubsec:surrogate}

We now turn the deterministic geometry $M_T(x)$ into an information channel by introducing a single linear Gaussian channel whose distinguishability is exactly $M_T(x)$.

\begin{definition}[Local Gaussian surrogate channel]
\label{def:surrogate-channel}
Let
$$
  A_T(x) := F_{\mathrm{sm}}(p(x))^{1/2}\,W J_T(x)\;\in\;\R^{V\times d_T},
$$
where $F_{\mathrm{sm}}(p(x))^{1/2}$, namely the \emph{Fisher factor}, is the positive
semidefinite square root of the softmax-Fisher matrix.
Fix a perturbation scale $\tau>0$, and let $U_T\sim\mathcal N(0,\tau\Id_{d_T})$ and $N\sim\mathcal N(0,\Id_V)$ be independent. The \emph{local Gaussian surrogate channel} associated with
$(x,T,\tau)$ is
\[
  Y_T = A_T(x)\,U_T + N \;\in\;\R^V.
\]
\end{definition}
\begin{remark}
The Fisher factor is built so that it reproduces the span geometry exactly.
Since
\begin{equation}
\label{eq:bridge}
\begin{aligned}
  & A_T(x)^\top A_T(x) \\
  &= J_T(x)^\top W^\top F_{\mathrm{sm}}(p(x))^{1/2} F_{\mathrm{sm}}(p(x))^{1/2} WJ_T(x) \\
  &= J_T(x)^\top W^\top F_{\mathrm{sm}}(p(x))\,WJ_T(x) = M_T(x),
\end{aligned}
\end{equation}
so $\|A_T(x)\,u\|_2^2=u^\top M_T(x)\,u$ is the Fisher-weighted output size of the
span perturbation $u$. Consequently the surrogate's per-perturbation divergence is
exactly the span quadratic form.
For any fixed $u$,
\begin{equation}
\label{eq:surrogate-shift-kl}
\begin{aligned}
  \KL\!\bigl(\mathcal N(A_T(x)u,\Id_V)\,\Vert\,\mathcal N(0,\Id_V)\bigr)
  & = \tfrac12\|A_T(x)u\|_2^2 \\
  & = \tfrac12\,u^\top M_T(x)\,u,
\end{aligned}
\end{equation}
matching the leading term of the true channel in Lemma~\ref{lem:local-kl}.
\end{remark}

Here $U_T$ is an isotropic local perturbation of the span and $\tau$ sets its scale.
The channel therefore measures how much of a span perturbation remains observable through its output-side effect at this resolution.

\paragraph{Faithfulness}
Equation~\eqref{eq:surrogate-shift-kl} matches the surrogate to the true channel only to leading order and only pointwise.
The following theorem, which is the key technical guarantee of the construction, upgrades this to a \emph{quantitative, direction-uniform} second-order agreement on a fixed neighborhood of the base prompt.
It is the precise sense in which the surrogate, and hence the capacity we derive from it, faithfully tracks the true autoregressive channel.

\begin{theorem}[Uniform second-order faithfulness]
\label{thm:faithfulness}
Under Assumption~\ref{assump:local-regularity}, there exist a radius $R>0$ and a function $\rho:\R\to\R$ with $\rho(\varepsilon)=o(\varepsilon^2)$ as
$\varepsilon\to0$ such that for every $\varepsilon\in\R$ and every $r\in\R^{d_T}$
with $\|r\|_2\le R$,
$$
  \Bigl|\,\KL\!\bigl(\pi_T(x_T+\varepsilon r)\,\Vert\,p(x)\bigr)
    -\tfrac12\,\varepsilon^2\,r^\top M_T(x)\,r\,\Bigr|
  \;\le\;|\rho(\varepsilon)| .
$$
By Equation~\eqref{eq:surrogate-shift-kl}, taking $u=\varepsilon r$, the next-token divergence agrees with the
surrogate's per-shift divergence uniformly to order $o(\varepsilon^2)$.
\end{theorem}
\begin{proof}
Set $a:=\|WJ\|_{\op}$, $w:=\|W\|_{\op}$, $L:=wK_T(x)$. We first prove a
channel-level cubic bound for a generic span perturbation and then instantiate it
at $u=\varepsilon r$.

\smallskip
\noindent
\textit{\underline{Step 1 (decompose the logit perturbation).}}
Fix the radius
\[
  R:=\min\{1,\,r_T(x)\}>0 .
\]
For a span perturbation $u$ with $\|u\|\le R$, write $\Delta:=\Delta(u)=s+t$ with
$s:=WJ\,u$ and $t:=W R_T(u)$, where $R_T(u):=h_T(x_T+u)-h_T(x_T)-J\,u$. Since
$\|u\|\le R\le r_T(x)$, Assumption~\ref{assump:local-regularity} gives
$\|R_T(u)\|\le K_T(x)\|u\|^2$, so
\begin{equation}
\label{eq:app-st-bounds}
  \|s\|\le a\|u\|,\qquad \|t\|\le w\|R_T(u)\|\le L\|u\|^2 .
\end{equation}
Moreover $R\le1$ ensures $\|u\|^2\le\|u\|$, hence
\begin{equation}
\label{eq:app-Delta-lin}
  \|\Delta\|\le\|s\|+\|t\|\le a\|u\|+L\|u\|^2\le(a+L)\|u\| .
\end{equation}

\smallskip
\noindent
\textit{\underline{Step 2 (logit-space cubic bound).}}
By Lemma~\ref{lem:softmax-cubic} applied to $\Delta$, and noting
$\softmax(z_0+\Delta)=\pi_T(x_T+u)$ and $\softmax(z_0)=p(x)$,
\begin{equation}
\label{eq:app-apply-a}
\begin{aligned}
  &\Bigl|\,\KL\!\bigl(\pi_T(x_T+u)\Vert p(x)\bigr)-\tfrac12\Delta^\top F\Delta\,\Bigr| \\
  &\qquad\le\tfrac23\|\Delta\|^3 .
\end{aligned}
\end{equation}

\smallskip
\noindent
\textit{\underline{Step 3 (replace $\Delta$ by $s$).}}
Expanding $\Delta=s+t$ and using symmetry of $F$,
$\tfrac12\Delta^\top F\Delta-\tfrac12 s^\top F s
=\tfrac12(2 s^\top F t+t^\top F t)$. By Lemma~\ref{lem:fisher-bilinear},
\begin{equation}
\label{eq:app-quad-diff}
  \Bigl|\tfrac12\Delta^\top F\Delta-\tfrac12 s^\top F s\Bigr|
  \le 2\|s\|\|t\|+\|t\|^2,
\end{equation}
and $s^\top F s=u^\top(WJ)^\top F(WJ)u=u^\top M u$, so $\tfrac12 s^\top F s
=\tfrac12 u^\top M u$ is the target quadratic.

\smallskip
\noindent
\textit{\underline{Step 4 (combine and collect powers).}}
Adding~\eqref{eq:app-apply-a} and~\eqref{eq:app-quad-diff},
\[
\begin{aligned}
  &\Bigl|\KL\!\bigl(\pi_T(x_T+u)\Vert p(x)\bigr)-\tfrac12 u^\top M u\Bigr| \\
  &\qquad\le\tfrac23\|\Delta\|^3+2\|s\|\|t\|+\|t\|^2 .
\end{aligned}
\]
Using~\eqref{eq:app-st-bounds},~\eqref{eq:app-Delta-lin}, and $\|u\|\le1$ (so
$\|u\|^4\le\|u\|^3$),
$$
  \|\Delta\|^3\le(a+L)^3\|u\|^3 ,
$$
$$
  \|s\|\|t\|\le aL\|u\|^3,\enspace\|t\|^2\le L^2\|u\|^4\le L^2\|u\|^3 ,
$$
so
\begin{equation}
\label{eq:app-channel-cubic}
\begin{aligned}
  \Bigl|\KL\!\bigl(\pi_T(x_T+u)\Vert p(x)\bigr)-\tfrac12 u^\top M u\Bigr| &\le C\|u\|^3 , \\
  \qquad \|u\|&\le R,
\end{aligned}
\end{equation}
with $C:=\tfrac23(a+L)^3+2aL+L^2\ge0$ and $L=wK_T(x)$.

\smallskip
\noindent
\textit{\underline{Step 5 (instantiate at $u=\varepsilon r$).}}
Fix a direction $r\in\R^{d_T}$ with $\|r\|_2\le R$ and a scale $\varepsilon\in\R$,
and apply the channel-level bound~\eqref{eq:app-channel-cubic} to the perturbation
$\varepsilon r$.
\begin{itemize}[leftmargin=*]
\item \textit{Bulk ($|\varepsilon|\,\|r\|\le R$).}
Then $\|\varepsilon r\|\le R$, so~\eqref{eq:app-channel-cubic} applies with
$u=\varepsilon r$; since $(\varepsilon r)^\top M(\varepsilon r)
=\varepsilon^2 r^\top M r$,
\[
\begin{aligned}
  &\Bigl|\KL\!\bigl(\pi_T(x_T+\varepsilon r)\Vert p(x)\bigr)
    -\tfrac12\varepsilon^2 r^\top M r\Bigr| \\
  &\qquad\le C\|\varepsilon r\|^3=C|\varepsilon|^3\|r\|^3\le C R^3|\varepsilon|^3 .
\end{aligned}
\]
\item \textit{Tail ($|\varepsilon|\,\|r\|>R$).}
Here $|\varepsilon|>R/\|r\|\ge1$ (using $\|r\|\le R$), so $\varepsilon^2\le|\varepsilon|^3$
and $1\le|\varepsilon|^3$. Both terms inside the absolute value are nonnegative, so
it is at most their sum. Since $p(x)$ has strictly positive entries,
$\KL(q\Vert p(x))\le B_x:=\max_i\log(1/p(x)_i)<\infty$ for every
$q\in\Delta^{V-1}$; and by Lemma~\ref{lem:fisher-bilinear},
$r^\top M r=(WJr)^\top F(WJr)\le 2\|WJr\|^2\le 2a^2\|r\|^2$, whence
$\tfrac12\varepsilon^2 r^\top M r\le a^2 R^2\varepsilon^2$. Therefore
\[
\begin{aligned}
  &\Bigl|\KL\!\bigl(\pi_T(x_T+\varepsilon r)\Vert p(x)\bigr)
    -\tfrac12\varepsilon^2 r^\top M r\Bigr| \\
  &\qquad\le B_x+a^2 R^2\varepsilon^2\le(B_x+a^2 R^2)\,|\varepsilon|^3 .
\end{aligned}
\]
\end{itemize}

\smallskip
\noindent
Setting $C_\star:=\max\{C R^3,\,B_x+a^2 R^2\}$ and
$\rho(\varepsilon):=C_\star\,|\varepsilon|^3$, both regimes give
\[
\begin{aligned}
  &\Bigl|\KL\!\bigl(\pi_T(x_T+\varepsilon r)\Vert p(x)\bigr)
    -\tfrac12\varepsilon^2 r^\top M r\Bigr| \\
  &\qquad\le\rho(\varepsilon) ,
\end{aligned}
\]
uniformly over $\|r\|_2\le R$, which is the bound asserted in
Theorem~\ref{thm:faithfulness} (note $\rho\ge0$, so $\rho=|\rho|$). Finally
$\rho(\varepsilon)/\varepsilon^2=C_\star|\varepsilon|\to0$, i.e.\
$\rho(\varepsilon)=o(\varepsilon^2)$ as $\varepsilon\to0$.
\end{proof}
\leancheck{GIF.realChannel\_higherOrderSurrogate\_ofLocalQuadraticControl}
\begin{remark}
The uniformity in Theorem~\ref{thm:faithfulness} is over the \emph{direction} $r$ on the ball $\|r\|_2\le R$:
a single envelope $\rho(\varepsilon)=o(\varepsilon^2)$ dominates the surrogate error simultaneously for all such $r$.
Thus the curvature $M_T(x)$ measured by the surrogate, and hence the capacity and influence derived from it below, tracks the true autoregressive channel on a fixed neighborhood of the base prompt, not merely in the infinitesimal limit of Lemma~\ref{lem:local-kl}.
\end{remark}

\subsubsection{The capacity measure and the influence proxy}
\label{subsubsec:capacity}

Having fixed a faithful surrogate, we read off its two scalar summaries and take the capacity as our measure.

\begin{definition}[Geometric scores]
\label{def:geometric-scores}
The \emph{local influence} and the \emph{local capacity} of span $T$ at prompt $x$
and scale $\tau>0$ are \looseness=-1
$$
  \mathrm{Inf}_T(x):=\tr M_T(x),
$$
$$
  \mathrm{Cap}_T(\tau;x)
    :=\tfrac12\log\det\!\bigl(\Id_{d_T}+\tau M_T(x)\bigr).
$$
\end{definition}

\begin{lemma}[Well-definedness of the capacity]
\label{lem:capacity-well-defined}
For every prompt $x$, span $T$, and scale $\tau>0$, the matrix
$\Id_{d_T}+\tau M_T(x)$ is positive definite; hence
$\det(\Id_{d_T}+\tau M_T(x))>0$ and $\mathrm{Cap}_T(\tau;x)$ is well-defined.
\end{lemma}
\begin{proof}
By Equation~\eqref{eq:bridge}, $M_T(x)=A_T(x)^\top A_T(x)\succeq0$, so $\forall v\neq0$,
$v^\top(\Id_{d_T}+\tau M_T(x))v=\|v\|_2^2+\tau\,v^\top M_T(x)\,v>0$. A
positive-definite matrix has positive determinant.
\end{proof}
\leancheckII{GIF.one\_add\_smul\_spanMetric\_posDef}{GIF.det\_one\_add\_smul\_spanMetric\_pos}

\smallskip
The identification rests on two standard facts about Gaussian laws.
\begin{lemma}[Gaussian quadratic-form moment]
\label{lem:moment}
For symmetric positive definite $\Sigma\in\R^{n\times n}$ and arbitrary
$M'\in\R^{n\times n}$,
$\;\E_{\xi\sim\mathcal N(0,\Sigma)}[\xi^\top M'\xi]=\tr(M'\Sigma).$
\end{lemma}
\begin{proof}
Using $\xi^\top M'\xi=\tr(M'\xi\xi^\top)$ and linearity of trace and expectation,
$\E[\xi^\top M'\xi]=\tr\!\bigl(M'\,\E[\xi\xi^\top]\bigr)=\tr(M'\Sigma)$, since
$\E[\xi\xi^\top]=\Sigma$ for $\xi\sim\mathcal N(0,\Sigma)$.
\end{proof}
\leancheck{GIF.integral\_quadratic\_form\_gaussianLawOfCov}

\begin{lemma}[Gaussian--Gaussian KL]
\label{lem:gkl}
For symmetric positive definite $\Sigma_1,\Sigma_2\in\R^{n\times n}$ and
$\mu\in\R^n$,
\[
\begin{aligned}
  &\KL\!\bigl(\mathcal N(\mu,\Sigma_1)\,\Vert\,\mathcal N(0,\Sigma_2)\bigr) \\
  &\quad=\tfrac12\Bigl(\log\tfrac{\det\Sigma_2}{\det\Sigma_1}
    +\tr(\Sigma_2^{-1}\Sigma_1)+\mu^\top\Sigma_2^{-1}\mu-n\Bigr).
\end{aligned}
\]
\end{lemma}
\begin{proof}
Let $\phi_1$ and $\phi_2$ be the Lebesgue densities of $\mathcal N(\mu,\Sigma_1)$
and $\mathcal N(0,\Sigma_2)$. Their log-ratio is
$$
  \log\frac{\phi_1(\xi)}{\phi_2(\xi)}
  =\tfrac12\log\frac{\det\Sigma_2}{\det\Sigma_1}
   -\tfrac12(\xi-\mu)^\top\Sigma_1^{-1}(\xi-\mu)
   +\tfrac12\xi^\top\Sigma_2^{-1}\xi .
$$
Taking the expectation under $\xi\sim\mathcal N(\mu,\Sigma_1)$ and applying
Lemma~\ref{lem:moment}: the centered term gives
$\E[(\xi-\mu)^\top\Sigma_1^{-1}(\xi-\mu)]=\tr(\Sigma_1^{-1}\Sigma_1)=n$, and using
$\E[\xi\xi^\top]=\Sigma_1+\mu\mu^\top$,
$\E[\xi^\top\Sigma_2^{-1}\xi]=\tr(\Sigma_2^{-1}\Sigma_1)+\mu^\top\Sigma_2^{-1}\mu$.
Substituting these into the expected log-ratio yields the claim.
\end{proof}
\leancheck{GIF.gaussianLawOfCovMean\_kl}

\begin{theorem}[The capacity is the surrogate mutual information]
\label{thm:cap-mi}
For the local Gaussian surrogate channel of Definition~\ref{def:surrogate-channel},
the mutual information transmitted from the span perturbation to the readout is
exactly the capacity:
$$
  \begin{aligned}
  I_{\mathrm{surr}}(U_T;Y_T)
  & := \KL\!\bigl(\mathcal L_{(U_T,Y_T)}\,\Vert\,\mathcal L_{U_T}\otimes\mathcal L_{Y_T}\bigr) \\
  & = \tfrac12\log\det\!\bigl(\Id_{d_T}+\tau M_T(x)\bigr)
  = \mathrm{Cap}_T(\tau;x).
  \end{aligned}
$$
\end{theorem}
\begin{proof}
Write $A:=A_T(x)=F^{1/2}WJ$, so that $A^\top A=M_T(x)$ by Equation~\eqref{eq:bridge}, and let $\Sigma_Y:=\Id_V+\tau AA^\top$ be the output covariance of $Y_T=AU_T+N$.
The pair $(U_T,Y_T)=(U_T,AU_T+N)$ is the image of the independent base
$\mathcal N(0,\Sigma_{\mathrm{base}})$, with $\Sigma_{\mathrm{base}}=\diag(\tau\Id_{d_T},\Id_V)$,
under the linear map $(u,n)\mapsto(u,Au+n)$; hence it is the centered Gaussian
$\mathcal N(0,\Sigma_{\mathrm{joint}})$ with
\[
  \Sigma_{\mathrm{joint}}
  =\begin{bmatrix}\tau\Id_{d_T} & \tau A^\top\\ \tau A & \Sigma_Y\end{bmatrix},
  \quad
  \Sigma_{\mathrm{prod}}
  =\begin{bmatrix}\tau\Id_{d_T} & 0\\ 0 & \Sigma_Y\end{bmatrix}.
\]
A centered Gaussian is determined by its covariance, and $\Sigma_{\mathrm{prod}}$ is
the joint covariance with its cross-blocks deleted, so
$\mathcal N(0,\Sigma_{\mathrm{prod}})=\mathcal L_{U_T}\otimes\mathcal L_{Y_T}$. Hence
$$
  I_{\mathrm{surr}}(U_T;Y_T)
  =\KL\!\bigl(\mathcal N(0,\Sigma_{\mathrm{joint}})\,\Vert\,\mathcal N(0,\Sigma_{\mathrm{prod}})\bigr),
$$
which by Lemma~\ref{lem:gkl} (centered case, $\mu=0$) equals
\[
  \tfrac12\Bigl(\log\tfrac{\det\Sigma_{\mathrm{prod}}}{\det\Sigma_{\mathrm{joint}}}
    +\tr(\Sigma_{\mathrm{prod}}^{-1}\Sigma_{\mathrm{joint}})-(d_T+V)\Bigr).
\]
By the Schur complement of the $(1,1)$ block and $\Sigma_Y-\tau AA^\top=\Id_V$,
\[
\begin{aligned}
  \det\Sigma_{\mathrm{joint}}
  &=\det(\tau\Id_{d_T})\,\det\!\bigl(\Sigma_Y-\tau A(\tau\Id_{d_T})^{-1}\tau A^\top\bigr) \\
  &=\det(\tau\Id_{d_T})\,\det(\Id_V),
\end{aligned}
\]
while $\det\Sigma_{\mathrm{prod}}=\det(\tau\Id_{d_T})\det\Sigma_Y$, so
$\det\Sigma_{\mathrm{prod}}/\det\Sigma_{\mathrm{joint}}=\det\Sigma_Y=\det(\Id_V+\tau AA^\top)$.
A direct block computation gives
$\Sigma_{\mathrm{prod}}^{-1}\Sigma_{\mathrm{joint}}
=\left[\begin{smallmatrix}\Id_{d_T}&A^\top\\[2pt] \ast&\Id_V\end{smallmatrix}\right]$,
whose diagonal blocks have trace $d_T+V$, cancelling the $-(d_T+V)$ term. Therefore
$I_{\mathrm{surr}}(U_T;Y_T)=\tfrac12\log\det(\Id_V+\tau AA^\top)$. Finally, by
Sylvester's determinant identity and the bridge~\eqref{eq:bridge},
\[
  \det(\Id_V+\tau AA^\top)=\det(\Id_{d_T}+\tau A^\top A)=\det(\Id_{d_T}+\tau M_T(x)),
\]
so $I_{\mathrm{surr}}(U_T;Y_T)=\tfrac12\log\det(\Id_{d_T}+\tau M_T(x))=\mathrm{Cap}_T(\tau;x)$.
\end{proof}
\leancheck{GIF.localEntropyCapacitySurrogate\_eq\_genuineChannelMutualInformation}
\begin{remark}
Both scores read off the spectrum of $M_T(x)$.
If $\lambda_1,\dots,\lambda_{d_T}\ge0$ are its eigenvalues, then
$$
  \mathrm{Inf}_T(x)=\sum_i\lambda_i,
  \enspace
  \mathrm{Cap}_T(\tau;x)=\tfrac12\sum_i\log(1+\tau\lambda_i).
$$
\end{remark}

\paragraph{Influence as a companion proxy}
The influence is the small-noise rate of the capacity.
Expanding $\log(1+\tau\lambda_i)=\tau\lambda_i+O(\tau^2)$ termwise gives
\begin{equation}
\label{eq:inf-rate}
  \mathrm{Cap}_T(\tau;x)=\tfrac{\tau}{2}\,\mathrm{Inf}_T(x)+O(\tau^2)
  \quad(\tau\to0^+),
\end{equation}
More precisely, this second-order correction is non-positive, since
$\log(1+z)\le z$ for $z\ge0$ (by concavity of $\log$). Therefore, $\forall \tau>0$,
$$
  \mathrm{Cap}_T(\tau;x)\le\tfrac{\tau}{2}\,\mathrm{Inf}_T(x) .
$$
Hence $\mathrm{Inf}_T(x)$ is simultaneously the leading-order rate
of the capacity and a global linear upper bound on it.
It is also a cheaper proxy while remaining exact to first order.

\begin{paperbox}{The Operational Measure of GIF}
We adopt the capacity as our operational measure of geometric information flow, and use influence as a cheaper first-order proxy, per Equation~\eqref{eq:inf-rate}.
\begin{equation}
\label{eq:gif-def}
  \mathrm{GIF}_T(\tau;x):=\mathrm{Cap}_T(\tau;x)
  \le\tfrac{\tau}{2}\,\mathrm{Inf}_T(x)
\end{equation}
By Theorem~\ref{thm:cap-mi} this is a genuine mutual information, and by Theorem~\ref{thm:faithfulness} it is computed from a surrogate that is faithful to the true channel.
\end{paperbox}

\subsection{Local Soundness}
\label{subsec:soundness-proof}

In \S\ref{subsec:gif-measurement} we defined the operational measure $\mathrm{GIF}_T(\tau;x)$ as the capacity of a local surrogate channel and justified it as a faithful approximation to the true autoregressive channel.
However, it remains to verify that it is conservative in the sense required of a leakage certificate, i.e., that it does not asymptotically undercount the true flow.
In this section, we show that our measure is locally sound per Equation~\eqref{eq:soundness}:
$$
I(U_T;Y)\;\le\;\mathrm{Cap}_T(\tau;x)+o(\tau)\qquad(\tau\to0^+) ,
$$ 
where $I(U_T;Y)$ is the true information flow of the genuine autoregressive channel, defined in Definition~\ref{def:geom-information-flow}.

\paragraph{Proof story}
Our insight is a single conservative relaxation: \textit{replacing the intractable optimal reference $P_Y$ by the computable base law $p(x)$ can only inflate the leakage, never deflate it}, which is exactly the safe direction for a certificate.
The proof has three moves. 
Concretely, as $\tau\to0^+$,
$$
\begin{aligned}
  I(U_T;Y)
    &= \E_{U_T}\,\KL\!\bigl(\pi_T(x_T+U_T)\,\big\Vert\,P_Y\bigr)
       && \text{[Thm.~\ref{thm:exact-form}]}\\
    &\le\; \E_{U_T}\,\KL\!\bigl(\pi_T(x_T+U_T)\,\big\Vert\,p(x)\bigr)
       && \text{[Move (i)]}\\
    &=\; \tfrac{\tau}{2}\,\mathrm{Inf}_T(x)+o(\tau)
       && \text{[Move (ii)]}\\
    &=\; \mathrm{Cap}_T(\tau;x)+o(\tau)
       && \text{[Move (iii)]} .
\end{aligned}
$$
We now unpack the three moves in detail.
\textbf{(i) Variational reduction.}
We replace the intractable output marginal $P_Y$ in the exact form (Theorem~\ref{thm:exact-form}) by the fixed base distribution $p(x)$, at the cost of a nonnegative slack. 
This turns the target into an averaged per-perturbation divergence against a fixed reference.
\textbf{(ii) Local averaging.}
We integrate this divergence over the perturbation law.
By the uniform faithfulness of the surrogate (Theorem~\ref{thm:faithfulness}), the per-perturbation divergence is the quadratic form of $M_T(x)$ up to a direction-uniform higher-order envelope, and averaging it against $U_T\sim\mathcal N(0,\tau\Id_{d_T})$ gives $\tfrac{\tau}{2}\mathrm{Inf}_T(x)+o(\tau)$.
\textbf{(iii) Capacity tightening.}
We replace the influence rate by the capacity itself, which differ only at order $\tau^2$. 

Note that, among the three moves, only the second is nontrivial and relies on the surrogate's faithfulness.
The first and third are elementary relaxations that hold for all $\tau>0$.
We give all three in full below, treating the elementary moves (i) and (iii) first and then the analytic core, move (ii).

\begin{lemma}[Variational reduction]
\label{lem:variational-reduction}
For every $\tau>0$,
$$
\begin{aligned}
  I(U_T&;Y) \\
  & = \E_{U_T}\!\bigl[\KL\bigl(\pi_T(x_T+U_T)\,\Vert\,p(x)\bigr)\bigr]
    - \KL\bigl(P_Y\,\Vert\,p(x)\bigr) \\
  & \le \E_{U_T}\!\bigl[\KL\bigl(\pi_T(x_T+U_T)\,\Vert\,p(x)\bigr)\bigr].
\end{aligned}
$$
\end{lemma}
\begin{proof}
For any strictly positive $q\in\Delta^{V-1}$, expand the divergence against $q$
through the marginal $P_Y$:
$$
\begin{aligned}
& \KL(\pi_T(x_T+U_T)\Vert q) = \\
& \KL(\pi_T(x_T+U_T)\Vert P_Y) +\sum_y\pi_T(x_T+U_T)_y\log\frac{P_Y(y)}{q(y)} .
\end{aligned}
$$
Taking $\E_{U_T}$ and using $\E_{U_T}[\pi_T(x_T+U_T)_y]=P_Y(y)$, the cross term
collapses to $\sum_y P_Y(y)\log\frac{P_Y(y)}{q(y)}=\KL(P_Y\Vert q)$, giving the
\emph{golden formula}
$$
  \E_{U_T}\!\bigl[\KL(\pi_T(x_T+U_T)\Vert q)\bigr]
  = I(U_T;Y)+\KL(P_Y\Vert q),
$$
where $I(U_T;Y)=\E_{U_T}[\KL(\pi_T(x_T+U_T)\Vert P_Y)]$ by Theorem~\ref{thm:exact-form}.
Setting $q=p(x)$ and dropping the nonnegative slack $\KL(P_Y\Vert p(x))\ge0$
(Gibbs' inequality~\cite{cover1999elements}) yields the claim.
\end{proof}
\leancheck{GIF.Softmax.softmax\_trueChannelMI\_le\_baseKL}

\begin{lemma}[Capacity tightening]
\label{lem:capacity-gap}
For every $\tau>0$, with $\lambda_1,\dots,\lambda_{d_T}\ge0$ the eigenvalues of
$M_T(x)$,
$$
\begin{aligned}
  0 
  & \le \tfrac{\tau}{2}\,\mathrm{Inf}_T(x)-\mathrm{Cap}_T(\tau;x) \\
  & = \tfrac12\sum_i\bigl(\tau\lambda_i-\log(1+\tau\lambda_i)\bigr)
  \le \tfrac{\tau^2}{4}\,\bigl(\mathrm{Inf}_T(x)\bigr)^2
  = O(\tau^2).
\end{aligned}
$$
\end{lemma}
\begin{proof}
Diagonalizing the positive semidefinite $M_T(x)$ with eigenvalues $\lambda_i\ge0$,
$$
\tfrac{\tau}{2}\mathrm{Inf}_T(x)-\mathrm{Cap}_T(\tau;x)
=\tfrac12\sum_i\bigl(\tau\lambda_i-\log(1+\tau\lambda_i)\bigr) .
$$
The elementary inequality $0\le t-\log(1+t)\le\tfrac12 t^2$ for $t\ge0$, applied to each
$t=\tau\lambda_i$ and summed, gives
$\tfrac12\sum_i\bigl(\tau\lambda_i-\log(1+\tau\lambda_i)\bigr)\le\tfrac{\tau^2}{4}\sum_i\lambda_i^2$.
Since the $\lambda_i\ge0$, $\sum_i\lambda_i^2\le\bigl(\sum_i\lambda_i\bigr)^2
=\bigl(\tr M_T(x)\bigr)^2=\bigl(\mathrm{Inf}_T(x)\bigr)^2$, yielding the quadratic remainder
$\tfrac{\tau^2}{4}\bigl(\mathrm{Inf}_T(x)\bigr)^2=O(\tau^2)$; nonnegativity follows from
$\log(1+t)\le t$.
\end{proof}
\leancheckII{GIF.Softmax.localCapacityProxy\_le\_half\_influence}{GIF.Softmax.cap\_gap\_upper}
\begin{remark}
The capacity and its small-noise influence rate agree to first order and separate only quadratically, so reporting the capacity moves the bound by $O(\tau^2)$.
\end{remark}

\paragraph{Local averaging (move (ii))}
As $\tau\to0^+$, the perturbation $U_T$ concentrates near the base prompt, where the next-token divergence is governed by the local curvature $M_T(x)$
(Lemma~\ref{lem:local-kl}).
The uniform faithfulness guarantee (Theorem~\ref{thm:faithfulness}) strengthens this pointwise expansion into a direction-uniform approximation, allowing us to average it over $U_T\sim\mathcal N(0,\tau\Id_{d_T})$.
The only remaining point is that \textit{the Gaussian tails do not disturb this limit}: large perturbations carry exponentially small Gaussian mass, while the divergence against the fixed base law $p(x)$ is uniformly bounded, so the tail contribution vanishes.
We make this precise. Abbreviate the per-perturbation divergence and its quadratic
remainder by
\[
\begin{aligned}
  D_x(u)&:=\KL\bigl(\pi_T(x_T+u)\,\Vert\,p(x)\bigr), \\
  r_x(u)&:=D_x(u)-\tfrac12\,u^\top M_T(x)\,u .
\end{aligned}
\]

\begin{lemma}[Gaussian averaging]
\label{lem:gauss-averaging}
Under Assumption~\ref{assump:local-regularity}, with $U_T\sim\mathcal N(0,\tau\Id_{d_T})$,
$$
  \E_{U_T}\!\bigl[D_x(U_T)\bigr]
  = \tfrac{\tau}{2}\,\mathrm{Inf}_T(x)+o(\tau)
  \qquad(\tau\to0^+).
$$
\end{lemma}
\begin{proof}
\emph{Pointwise expansion.} By Lemma~\ref{lem:local-kl},
$r_x(u)=o(\|u\|_2^2)$ as $u\to0$; in particular there is $\delta>0$ with
$|r_x(u)|\le\|u\|_2^2$ whenever $\|u\|_2<\delta$.

\smallskip
\noindent
\textit{\underline{Global quadratic domination.}}
Since $p(x)$ has strictly positive entries, the divergence is uniformly bounded: for any $q\in\Delta^{V-1}$,
$$\KL(q\Vert p(x))=\sum_i q_i\log\frac{q_i}{p(x)_i}\le\sum_i q_i\log\frac1{p(x)_i}\le B_x ,$$
where $B_x:=\max_i\log(1/p(x)_i)<\infty$ (using $\log q_i\le0$). Hence
$|r_x(u)|\le\max\{D_x(u),\,\tfrac12 u^\top M_T(x)u\}
\le B_x+\tfrac12\|M_T(x)\|_{\op}\|u\|_2^2$ for all $u$. Combining the two regimes
(for $\|u\|_2\ge\delta$ bound $B_x\le B_x\|u\|_2^2/\delta^2$),
\begin{equation}
\label{eq:global-dom}
\begin{aligned}
  & |r_x(u)|\le C_x\,\|u\|_2^2
  \quad\text{for all }u\in\R^{d_T}, \\
  & C_x:=\max\!\Bigl\{1,\;\tfrac{B_x}{\delta^2}+\tfrac12\|M_T(x)\|_{\op}\Bigr\}.
\end{aligned}
\end{equation}

\smallskip
\noindent
\textit{\underline{Dominated convergence.}} Write $U_T=\sqrt{\tau}\,G$ with $G\sim\mathcal N(0,\Id_{d_T})$.
Then
$$
\begin{aligned}
  & \frac{1}{\tau}\,\E_{U_T}\!\bigl[D_x(U_T)\bigr] \\
  & = \tfrac12\,\E_G\!\bigl[G^\top M_T(x)\,G\bigr]
    + \E_G\!\left[\frac{r_x(\sqrt{\tau}\,G)}{\tau}\right].
\end{aligned}
$$
For the first term, $\E_G[G^\top M_T(x)G]=\tr M_T(x)=\mathrm{Inf}_T(x)$ by
Lemma~\ref{lem:moment}. For the second, the integrand vanishes pointwise as
$\tau\to0^+$, since for fixed $G$,
$$
  r_x(\sqrt\tau G)/\tau=\bigl(r_x(\sqrt\tau G)/\|\sqrt\tau G\|_2^2\bigr)\|G\|_2^2\to0
$$
by the little-$o$ property; and by Equation~\eqref{eq:global-dom} it is dominated by
$$|r_x(\sqrt\tau G)/\tau|\le C_x\|G\|_2^2 ,$$
which is integrable ($\E_G\|G\|_2^2=d_T$). Dominated convergence gives
$\E_G[r_x(\sqrt\tau G)/\tau]\to0$, i.e.\ $\E_{U_T}[r_x(U_T)]=o(\tau)$. Therefore
$\E_{U_T}[D_x(U_T)]=\tfrac{\tau}{2}\mathrm{Inf}_T(x)+o(\tau)$.
\end{proof}
\leancheck{GIF.localKLInterpretation\_affineHead\_centeredGaussian\_ofLocalQuadraticControl}
\begin{theorem}[Local soundness of GIF]
\label{thm:local-soundness}
Under Assumption~\ref{assump:local-regularity}, the operational measure is locally
sound: Equation~\eqref{eq:soundness} holds with
$\mathrm{GIF}_T(\tau;x)=\mathrm{Cap}_T(\tau;x)$, i.e.
\[
  I(U_T;Y)\;\le\;\mathrm{GIF}_T(\tau;x)+o(\tau)\qquad(\tau\to0^+).
\]
\end{theorem}
\begin{proof}
Chaining the three moves,
$$
\begin{aligned}
  I(U_T;Y)
    &\le\; \E_{U_T}\!\bigl[D_x(U_T)\bigr]
       && \text{[Lem.~\ref{lem:variational-reduction}]}\\
    &=\; \tfrac{\tau}{2}\,\mathrm{Inf}_T(x)+o(\tau)
       && \text{[Lem.~\ref{lem:gauss-averaging}]}\\
    &=\; \mathrm{Cap}_T(\tau;x)+o(\tau)
       && \text{[Lem.~\ref{lem:capacity-gap}]} ,
\end{aligned}
$$
where the last step absorbs the $O(\tau^2)$ capacity gap into $o(\tau)$. Since
$\mathrm{GIF}_T(\tau;x)=\mathrm{Cap}_T(\tau;x)$, this is Equation~\eqref{eq:soundness}.
\end{proof}
\leancheck{GIF.Softmax.softmax\_trueChannelMI\_le\_capacity\_add\_littleo}

\section{Implementation}
\label{sec:implementation}
While~\autoref{sec:theory} derived a sound, operational measure of GIF, direct evaluation is prohibitively expensive at scale. This section bridges theory and practice: we first identify the computational bottleneck and derive an efficient estimator, then describe the full system implementation.

\subsection{Efficient Flow Approximation}
\label{subsec:efficient}
\paragraph{The bottleneck}
As discussed in~\autoref{subsec:gif-measurement}, computing GIF reduces to forming $M_T(x)$, which in turn requires materializing $J_T(x)$.
Via reverse-mode automatic differentiation, this demands one backward pass per output dimension $m=d_{\mathrm{model}}$~\cite{baydin2018automatic}.
For modern open-weight LLMs, $m$ runs into the thousands, ranging from 2880 for \gptosslarge~\cite{gptoss120bconfig} to 5376 for \gemma~\cite{gemma4hfconfig} (the models we evaluate), and up to 7168 for the largest open-weight models like DeepSeek V4 Pro~\cite{deepseekv4proconfig}, making exact GIF computation orders of magnitude more expensive than inference.

\paragraph{Our insight}
Recall that the capacity $\mathrm{Cap}_T(\tau;x)$ is the operational GIF score, whereas the influence $\mathrm{Inf}_T(x)$ is its leading-order proxy in~\autoref{eq:gif-def}.
The key observation is that $\mathrm{Inf}_T(x)$ is a trace and \textit{a trace can be estimated without explicitly constructing the matrix}.
We therefore adopt randomized trace estimation~\cite{hutchinson1989stochastic}, reducing the computation to a small number of probe evaluations used by the estimator.


\paragraph{Our approach}
We estimate the influence $\mathrm{Inf}_T(x)$ with Hutch\texttt{++}~\cite{meyer2021hutch++}.
By the standard Hutchinson trace identity~\cite{hutchinson1989stochastic}, for a random vector $z$ with $\E[zz^\top]=\Id_{d_T}$,
$$
  \mathrm{Inf}_T(x)=\tr M_T(x)=\E_z\!\bigl[z^\top M_T(x)\,z\bigr] ,
$$
where the vector $z$ is typically called a random probe.
The trace can therefore be estimated from scalar quadratic forms $z^\top M_T(x)\,z=z^\top\xi$, where 
$$\xi=M_T(x)\,z=J_T(x)^\top W^\top F_{\mathrm{sm}}(p(x))\,W J_T(x)\,z .$$ 
$\xi$ can be computed by pushing probe $z$ through the network and back at the fixed prompt $x$, with $p:=p(x)$: 
$$\xi_1 = J_T(x)\,z,\quad\xi_2 = W\,\xi_1,\quad\xi_3 = p\odot\xi_2-p\,(p^\top\xi_2) ,$$
$$\xi_4 = W^\top\xi_3,\quad\xi = J_T(x)^\top\xi_4 .$$
Only the first and last steps involve the model.
The first step, $\xi_1=J_T(x)\,z$, is the directional derivative of the span-to-hidden map along $z$, which can be computed either by a forward-mode pass or by the standard double-backward construction.
The last step, $\xi=J_T(x)^\top\xi_4$, is the gradient of $h_T(x)^\top\xi_4$ with respect to the span embeddings, which can be computed by one backward pass.
The middle three steps are cheap logit-space algebra and require no model pass.
Therefore, we estimate $\mathrm{Inf}_T(x)$ by querying the matrix-vector oracle $z \mapsto M_T(x)z$ a small number $\ell$ of times, obtaining an estimate $\widehat{\mathrm{Inf}}_T$,
$$\widehat{\mathrm{Inf}}_T = \frac{1}{\ell}\sum_{i=1}^\ell z_i^\top M_T(x)\,z_i = \frac{1}{\ell}\sum_{i=1}^\ell z_i^\top\xi^{(i)}.$$

\paragraph{Probabilistic soundness}
Note that our system reports the estimate $\widehat{\mathrm{Inf}}_T$ rather than the exact influence.
The deterministic guarantee of Theorem~\ref{thm:local-soundness} must then be restated to account for estimation error.
Since $M_T(x)\succeq0$ by Lemma~\ref{lem:capacity-well-defined}, Hutch\texttt{++} gives a \emph{multiplicative} trace estimate~\cite{meyer2021hutch++}.
More specifically, when Hutch\texttt{++} is run with Gaussian probes
$z\sim\mathcal{N}(0,I_{D_T})$, its standard unbiasedness and variance-bound analysis implies that,
for any target accuracy $\varepsilon\in(0,1)$ and failure probability $\delta\in(0,1)$, using $\ell \;\ge\; 2+4/(\varepsilon\sqrt{\delta})$ matrix-vector oracle queries returns an estimate $\widehat{\mathrm{Inf}}_T$ satisfying
$$
(1-\varepsilon)\,\mathrm{Inf}_T(x)\le\widehat{\mathrm{Inf}}_T\le(1+\varepsilon)\,
\mathrm{Inf}_T(x)
$$ 
with probability at least $1-\delta$.
We \emph{adopt} this Hutch\texttt{++} guarantee (its unbiasedness and variance bound~\cite{meyer2021hutch++}) rather than reprove it; the reduction from that guarantee to the displayed $(\varepsilon,\delta)$ estimate, by Chebyshev's inequality, is machine-checked in Lean~4 with the guarantee supplied as a hypothesis.
\leancheck{GIF.hutchpp\_estimate\_relativeError\_failure\_le}
Composing this estimate with Theorem~\ref{thm:local-soundness} yields a probabilistic soundness guarantee.

\begin{corollary}[Probabilistic soundness under estimation]
\label{cor:estimated-soundness}
Under Assumption~\ref{assump:local-regularity}, for any $\varepsilon,\delta\in(0,1)$ and any span with nonzero influence $\mathrm{Inf}_T(x)>0$ (the case $\mathrm{Inf}_T(x)=0$ is the trivial no-flow case, where the span has no first-order effect on the output), the Hutch\texttt{++} estimate $\widehat{\mathrm{Inf}}_T$ obtained with
$\ell\ge2+4/(\varepsilon\sqrt{\delta})$
matrix-vector oracle queries satisfies, with probability at least $1-\delta$,
\[
  I(U_T;Y)
  \;\le\;
  \frac{\tau}{2}\cdot\frac{\widehat{\mathrm{Inf}}_T}{1-\varepsilon}
  \;+\; o(\tau)
  \qquad (\tau\to0^+).
\]
\end{corollary}
\begin{proof}
By Theorem~\ref{thm:local-soundness} and the global bound
$\mathrm{Cap}_T(\tau;x)\le\tfrac{\tau}{2}\mathrm{Inf}_T(x)$ (Equation~\eqref{eq:gif-def}),
the true flow obeys $I(U_T;Y)\le\tfrac{\tau}{2}\mathrm{Inf}_T(x)+o(\tau)$ as $\tau\to0^+$.
The Hutch\texttt{++} guarantee gives, with probability at least $1-\delta$,
$\widehat{\mathrm{Inf}}_T\ge(1-\varepsilon)\,\mathrm{Inf}_T(x)$, i.e.\
$\mathrm{Inf}_T(x)\le\widehat{\mathrm{Inf}}_T/(1-\varepsilon)$. Substituting into the
previous bound yields, on the same event,
$$
  I(U_T;Y)\;\le\;\frac{\tau}{2}\cdot\frac{\widehat{\mathrm{Inf}}_T}{1-\varepsilon}\;+\;o(\tau)
  \qquad(\tau\to0^+),
$$
as claimed.
\end{proof}
\leancheckN{GIF.hutchpp\_estimatedSoundness}{Composes the relative-error estimate with the local-soundness Theorem~\ref{thm:local-soundness}; both are conditional on the cited Hutch\texttt{++} guarantee, which we adopt as a hypothesis because Hutch\texttt{++} has no Lean~4 proof.}

\subsection{System Details}
\label{subsec:impl}
This section describes our implementation of GIF-based information-flow control. Our system operates alongside LLM inference in an agentic system. It has three components: 
(1) a labeling component that assigns security labels to tokens; 
(2) a flow-tracking component that estimates GIF scores between labeled tokens; 
and (3) a declassifier that makes the final security decision based on the estimated flow and other contextual evidence.

\paragraph{Labeling}
Following standard IFC practice, we label each token using a two-point lattice. For confidentiality, \texttt{Low} denotes public data and \texttt{High} denotes confidential data; a High-to-Low dependency is a potential leak. For integrity, \texttt{Low} denotes untrusted data and \texttt{High} denotes trusted data or privileged actions; a Low-to-High dependency is a potential integrity violation, whereas a High-to-Low violation is a potential confidentiality violation. 
Similar to prior IFC-based defenses for agentic systems~\cite{christodorescu2025systems,wu2024system, costa2025securing, debenedetti2025camel, beurerkellner2025design, zhong2025rtbasdefendingllmagents}, we assign labels using heuristics.
Although GIF can support more sophisticated labeling policies, their development is orthogonal to our work; we therefore leave them to future work and focus on flow tracking.

\paragraph{Flow tracking}
This component computes the GIF score between a sink token (\texttt{High} for integrity and \texttt{Low} for confidentiality tasks, respectively) and each token in its preceding context. We instrument the autoregressive generation process and intervene whenever the model generates such a token. 
At that point we apply the efficient Hutch\texttt{++} approximation described in~\autoref{subsec:efficient}. The output is a ranked list of prefix tokens with their GIF scores relative to the sink token.
\looseness=-1

\paragraph{Declassifier}
The declassifier consumes token-level policy labels together with a policy engine's token-attribution signal and produces the final security decision. 

\begin{itemize}[leftmargin=10pt]
\item \textbf{Oracle threshold declassifier.}
A natural quantitative-IFC design denies when estimated flow from policy-relevant sources to generated sinks exceeds a threshold~\cite{backes2009automatic,chatzikokolakis2010statistical,chothia2011statistical,smith2009foundations}. Since attribution scores are not calibrated across policy engines, we use rankings rather than raw scores. For each trajectory, \textit{Policy Source Precision} at k ($\textsc{PSP@}k$) measures the fraction of top-$k$ source tokens that overlap benchmark-labeled prompt-injection or private-information tokens. We term $k$ \textit{attribution cutoff}. A concrete threshold declassifier would deny when $\textsc{PSP@}k \geq \tau$. To avoid bias from choosing $\tau$, we report AUROC and AUPRC over all thresholds, measuring separation between violating and non-violating trajectories without assuming score calibration.
    
    \item \textbf{LLM-as-a-declassifier.}
The oracle tests whether \gif surfaces the right tokens, but not whether their influence is harmful: untrusted content may be benign task data, and private information may be necessary for the user's request. We therefore evaluate an LLM declassifier. For each candidate tool call or action, GIF selects the top-$k$ policy-relevant source spans influencing the generated sink tokens, embeds them in minimal task context, marks them with \texttt{[[HIGHLIGHT]]...[[/HIGHLIGHT]]}, and passes this reduced evidence, the user task, and the target action to the LLM declassifier. 
The LLM returns \textsc{Allow}/\textsc{Deny}; a trajectory violates policy if any candidate action is denied.
\end{itemize}
This design does not replace formal policy semantics; it tests whether GIF supplies compact, policy-relevant evidence to a semantic decision procedure. Developing a formal declassifier semantics for agentic systems is complementary and left to future work.

\section{Evaluation}
\label{sec:evaluation}

\paragraph{Datasets and tasks}
We evaluate prompt injection as integrity violations and privacy leakage as confidentiality violations in agentic systems. For integrity, we use AgentDojo~\cite{debenedetti2024agentdojo} and MSB~\cite{zhang2026mcp}, which test prompt-injection attacks against tool-using agents. For confidentiality, we use AgentDAM~\cite{zharmagambetov2026agentdam}, which tests privacy leakage in WebArena-derived browser tasks~\cite{zhou2024webarena,koh-etal-2024-visualwebarena} across GitLab, Reddit, and shopping environments. These benchmarks are widely used in recent agent-security work~\cite{costa2025securing,debenedetti2025camel,zhong2025rtbasdefendingllmagents,shi2026progentsecuringaiagents,alizadeh2025simplepromptinjectionattacks,el2026agentleak}.

\paragraph{Backbone models for agent trajectories}
We generate trajectories with each benchmark's native harness: the MCP tool-calling agent for MSB, the WebArena browser-agent scaffolding for AgentDAM, and the AgentDojo tool-calling pipeline. We instantiate each with six open-weight reasoning-capable models: \qwensmall, \qwenmedium, \qwenlarge, \gemma, \gptosssmall, and \gptosslarge, using reasoning mode and recommended sampling parameters. We use open-weight models because GIF requires model internals, and to test whether flows from smaller surrogates transfer across model sizes and families.

\paragraph{Ground-truth labels}
We use each benchmark's integrity/confidentiality oracle as trajectory-level ground truth. AgentDojo and MSB use deterministic checkers for attacker success; AgentDAM uses sensitive-data annotations and an LLM-judge to detect disclosure of protected information. We align benchmark-provided prompt-injection and private-information annotations to model input tokens by string matching, with LLM-assisted disambiguation when needed.

\paragraph{Machine details}
Experiments ran on 3 NVIDIA RTX PRO 6000 GPUs with 48 vCPUs and 565 GB RAM, using CUDA 13.2, PyTorch 2.12, vLLM 0.22.0, and HF transformers 5.9.0.
\looseness=-1

\begin{table*}[t!]
\centering
\caption{GIF performance with deterministic declassifier, swept over PSP@5--PSP@50 (shown as @5--@50). Per (dataset, model) block and metric column the {\setlength{\fboxsep}{1pt}\colorbox{bestc}{best}} value is green and the {\setlength{\fboxsep}{1pt}\colorbox{worstc}{worst}} red. Values are rounded to two decimals (leading zero dropped) to save space. Models: GPT-A/B $=$ GPT-OSS-120B/20B; Gemma4 $=$ Gemma-4-31B; Qwen-A/B/C $=$ Qwen3-32B/14B/8B.}
\label{tab:token-level}
\scriptsize
\setlength{\tabcolsep}{1.4pt}
\renewcommand{\arraystretch}{1.1}
\begin{tabular}{@{}cl*{10}{c}@{\hspace{3pt}}*{10}{c}@{\hspace{3pt}}*{10}{c}@{}}
\toprule
 & \multirow{3}{*}{\textbf{Method}} & \multicolumn{10}{c}{\textbf{AgentDojo (Integrity)}} & \multicolumn{10}{c}{\textbf{MSB (Integrity)}} & \multicolumn{10}{c}{\textbf{AgentDAM (Confidentiality)}} \\
\cmidrule(lr){3-12}\cmidrule(lr){13-22}\cmidrule(lr){23-32}
 & & \multicolumn{5}{c}{AUROC} & \multicolumn{5}{c}{AUPRC} & \multicolumn{5}{c}{AUROC} & \multicolumn{5}{c}{AUPRC} & \multicolumn{5}{c}{AUROC} & \multicolumn{5}{c}{AUPRC} \\
\cmidrule(lr){3-7}\cmidrule(lr){8-12}\cmidrule(lr){13-17}\cmidrule(lr){18-22}\cmidrule(lr){23-27}\cmidrule(lr){28-32}
 & & @5 & @10 & @15 & @30 & @50 & @5 & @10 & @15 & @30 & @50 & @5 & @10 & @15 & @30 & @50 & @5 & @10 & @15 & @30 & @50 & @5 & @10 & @15 & @30 & @50 & @5 & @10 & @15 & @30 & @50 \\
\midrule
\multirow{3}{*}{\rotatebox[origin=c]{90}{GPT-A}}
 & GIF        & \B{.79} & \B{.83} & \B{.85} & \B{.88} & \B{.88} & \B{.62} & \B{.69} & \B{.74} & \B{.77} & .73 & \B{.90} & \B{.91} & \B{.93} & \B{.97} & .97 & \B{.95} & \B{.95} & \B{.96} & \B{.98} & \B{.98} & \B{.69} & \B{.70} & \B{.74} & \B{.79} & \B{.83} & \B{.41} & \B{.40} & \B{.43} & \B{.49} & \B{.52} \\
 & GIF$^{-}$  & .77 & .81 & .83 & .87 & \B{.88} & .61 & .66 & .70 & .75 & \B{.76} & .89 & \B{.91} & \B{.93} & \B{.97} & \B{.98} & .94 & \B{.95} & \B{.96} & \B{.98} & \B{.98} & .66 & .66 & .65 & .67 & .78 & .25 & .22 & .20 & \W{.10} & \W{.15} \\
 & RTBAS      & \W{.51} & \W{.51} & \W{.52} & \W{.62} & \W{.69} & \W{.24} & \W{.24} & \W{.25} & \W{.32} & \W{.35} & \W{.51} & \W{.53} & \W{.58} & \W{.71} & \W{.80} & \W{.75} & \W{.76} & \W{.79} & \W{.85} & \W{.90} & \W{.50} & \W{.50} & \W{.47} & \W{.55} & \W{.52} & \W{.18} & \W{.18} & \W{.17} & .19 & .18 \\
\cmidrule(lr){2-32}
\multirow{3}{*}{\rotatebox[origin=c]{90}{GPT-B}}
 & GIF        & \B{.77} & \B{.85} & \B{.87} & \B{.91} & .92 & \B{.45} & \B{.57} & \B{.61} & \B{.64} & .65 & \B{.76} & \B{.80} & .82 & .85 & \B{.85} & \B{.81} & \B{.85} & \B{.85} & .87 & .88 & .67 & .71 & .72 & \B{.78} & \B{.84} & .37 & .43 & \W{.45} & .50 & .52 \\
 & GIF$^{-}$  & .75 & .83 & \B{.87} & \B{.91} & \B{.93} & .44 & .55 & .60 & .63 & \B{.66} & .74 & .78 & \B{.83} & \B{.86} & \B{.85} & .79 & .82 & \B{.85} & \B{.88} & \B{.89} & \W{.62} & \W{.69} & \W{.70} & \B{.78} & \B{.84} & \W{.31} & \W{.42} & .46 & \B{.53} & \B{.58} \\
 & RTBAS      & \W{.74} & \W{.79} & \B{.87} & \W{.88} & \W{.88} & \W{.42} & \W{.45} & \W{.55} & \W{.57} & \W{.58} & \W{.62} & \W{.68} & \W{.73} & \W{.74} & \W{.74} & \W{.72} & \W{.76} & \W{.78} & \W{.78} & \W{.79} & \B{.69} & \B{.75} & \B{.75} & \W{.72} & \W{.72} & \B{.43} & \B{.51} & \B{.51} & \W{.45} & \W{.42} \\
\cmidrule(lr){2-32}
\multirow{3}{*}{\rotatebox[origin=c]{90}{Gemma4}}
 & GIF        & \B{.56} & .62 & .65 & .75 & .80 & .19 & .24 & .26 & .34 & .38 & \B{.81} & \B{.83} & \B{.84} & \B{.87} & \B{.88} & \B{.81} & \B{.82} & \B{.83} & \B{.85} & \W{.85} & \B{.59} & .58 & .60 & .59 & .61 & \W{.13} & \W{.12} & .14 & \W{.13} & \W{.14} \\
 & GIF$^{-}$  & \W{.55} & \W{.60} & \W{.62} & \W{.72} & \W{.77} & \W{.18} & \W{.22} & \W{.23} & \W{.32} & \W{.35} & .80 & .81 & .83 & \B{.87} & \B{.88} & .79 & .80 & .82 & \W{.84} & .86 & .58 & \B{.62} & \B{.61} & \B{.62} & \B{.63} & .14 & .14 & \W{.13} & \W{.13} & \W{.14} \\
 & RTBAS      & \B{.56} & \B{.72} & \B{.74} & \B{.81} & \B{.94} & \B{.22} & \B{.40} & \B{.47} & \B{.57} & \B{.79} & \W{.51} & \W{.56} & \W{.70} & \W{.81} & \B{.88} & \W{.59} & \W{.65} & \W{.76} & \B{.85} & \B{.90} & \W{.50} & \W{.50} & \W{.50} & \W{.50} & \W{.50} & \B{.15} & \B{.15} & \B{.15} & \B{.15} & \B{.15} \\
\cmidrule(lr){2-32}
\multirow{3}{*}{\rotatebox[origin=c]{90}{Qwen-A}}
 & GIF        & \B{.81} & \B{.88} & \B{.92} & .94 & \B{.96} & \B{.61} & \B{.71} & \B{.74} & .78 & \B{.81} & \B{.82} & .84 & .86 & \B{.89} & \B{.89} & \B{.92} & \B{.93} & .94 & \B{.96} & \B{.96} & \B{.61} & \B{.65} & \B{.69} & .71 & \B{.72} & \B{.47} & \B{.49} & \B{.52} & \B{.54} & \B{.54} \\
 & GIF$^{-}$  & .79 & .87 & .91 & \B{.95} & \B{.96} & .58 & .69 & \B{.74} & \B{.80} & \B{.81} & \B{.82} & \B{.85} & \B{.87} & .88 & \B{.89} & \B{.92} & \B{.93} & \B{.95} & .95 & \B{.96} & .58 & .63 & .67 & \B{.72} & .71 & \W{.45} & .48 & .51 & \B{.54} & .53 \\
 & RTBAS      & \W{.51} & \W{.56} & \W{.62} & \W{.74} & \W{.80} & \W{.18} & \W{.25} & \W{.30} & \W{.48} & \W{.56} & \W{.56} & \W{.57} & \W{.59} & \W{.70} & \W{.73} & \W{.78} & \W{.79} & \W{.80} & \W{.86} & \W{.87} & \W{.53} & \W{.54} & \W{.55} & \W{.56} & \W{.59} & \W{.45} & \W{.46} & \W{.47} & \W{.46} & \W{.48} \\
\cmidrule(lr){2-32}
\multirow{3}{*}{\rotatebox[origin=c]{90}{Qwen-B}}
 & GIF        & \B{.78} & \B{.87} & \B{.90} & \B{.96} & \B{.97} & \B{.61} & \B{.75} & \B{.80} & \B{.88} & \B{.89} & \B{.91} & \B{.95} & \B{.97} & \B{.98} & .97 & \B{.92} & \B{.96} & \B{.97} & \B{.98} & .97 & \B{.63} & \B{.69} & .68 & .73 & \B{.75} & \B{.43} & \B{.43} & \B{.43} & \B{.46} & \B{.47} \\
 & GIF$^{-}$  & .76 & .83 & .87 & .94 & .96 & .57 & .68 & .74 & .86 & .85 & .90 & .94 & \B{.97} & .97 & \B{.99} & .90 & .94 & \B{.97} & \B{.98} & \B{.99} & .61 & .66 & \B{.70} & \B{.74} & .73 & .38 & .41 & \B{.43} & \B{.46} & .46 \\
 & RTBAS      & \W{.51} & \W{.60} & \W{.72} & \W{.80} & \W{.88} & \W{.19} & \W{.33} & \W{.52} & \W{.63} & \W{.75} & \W{.49} & \W{.49} & \W{.51} & \W{.66} & \W{.71} & \W{.53} & \W{.53} & \W{.54} & \W{.68} & \W{.73} & \W{.50} & \W{.50} & \W{.50} & \W{.49} & \W{.56} & \W{.31} & \W{.31} & \W{.31} & \W{.31} & \W{.34} \\
\cmidrule(lr){2-32}
\multirow{3}{*}{\rotatebox[origin=c]{90}{Qwen-C}}
 & GIF        & \B{.76} & \B{.85} & .87 & \B{.92} & \B{.93} & .44 & \B{.57} & .61 & \B{.69} & .72 & .85 & .89 & .92 & \B{.95} & \B{.95} & .88 & .92 & .94 & \B{.96} & \B{.96} & \B{.65} & \B{.69} & \B{.71} & .74 & \B{.75} & .39 & .41 & .42 & .44 & .43 \\
 & GIF$^{-}$  & \B{.76} & .83 & \B{.89} & .91 & \B{.93} & \B{.46} & .55 & \B{.64} & \B{.69} & \B{.73} & \B{.89} & \B{.92} & \B{.94} & .94 & \B{.95} & \B{.91} & \B{.93} & \B{.95} & \B{.96} & \B{.96} & .61 & .68 & \B{.71} & \B{.76} & \B{.75} & \B{.40} & \B{.44} & \B{.44} & \B{.47} & \B{.44} \\
 & RTBAS      & \W{.51} & \W{.59} & \W{.67} & \W{.72} & \W{.83} & \W{.12} & \W{.23} & \W{.35} & \W{.36} & \W{.46} & \W{.49} & \W{.52} & \W{.57} & \W{.68} & \W{.73} & \W{.59} & \W{.60} & \W{.63} & \W{.71} & \W{.73} & \W{.50} & \W{.50} & \W{.50} & \W{.52} & \W{.52} & \W{.28} & \W{.28} & \W{.28} & \W{.29} & \W{.30} \\
\midrule
\multirow{3}{*}{\rotatebox[origin=c]{90}{\textbf{Average}}}
 & GIF        & \B{.75} & \B{.82} & \B{.85} & \B{.89} & \B{.91} & \B{.49} & \B{.59} & \B{.63} & \B{.69} & \B{.70} & \B{.84} & \B{.87} & .89 & \B{.92} & \B{.92} & \B{.88} & \B{.90} & .91 & \B{.93} & .93 & \B{.64} & \B{.67} & \B{.69} & \B{.72} & \B{.75} & \B{.37} & \B{.38} & \B{.40} & \B{.43} & \B{.44} \\
 & GIF$^{-}$  & .73 & .79 & .83 & \B{.89} & .90 & .47 & .56 & .61 & .68 & .69 & \B{.84} & \B{.87} & \B{.90} & \B{.92} & \B{.92} & \B{.88} & \B{.90} & \B{.92} & \B{.93} & \B{.94} & .61 & .66 & .67 & .71 & .74 & .32 & .35 & .36 & .37 & .38 \\
 & RTBAS      & \W{.56} & \W{.63} & \W{.69} & \W{.76} & \W{.84} & \W{.23} & \W{.32} & \W{.41} & \W{.49} & \W{.58} & \W{.53} & \W{.56} & \W{.61} & \W{.72} & \W{.76} & \W{.66} & \W{.68} & \W{.72} & \W{.79} & \W{.82} & \W{.54} & \W{.55} & \W{.55} & \W{.56} & \W{.57} & \W{.30} & \W{.32} & \W{.32} & \W{.31} & \W{.31} \\
\bottomrule
\end{tabular}
\end{table*}

\subsection{RQ1: How well does \gif surface confidentiality and integrity violations?}
\label{sec:rq1}
\paragraph{Setup}
We first evaluate \gif as a policy signal in isolation. Since existing agent benchmarks do not provide ground-truth causal influence labels for every generated tool call/action, we use an oracle threshold-declassifier proxy. The oracle knows which input tokens correspond to policy-relevant content: prompt-injection instructions for integrity attacks, and private-information tokens for confidentiality attacks. A useful policy engine should rank these gold tokens substantially higher in trajectories where the agent violates policy than in trajectories where it does not.

To make the comparison independent of score calibration, we evaluate each policy engine only through the ranking it induces over source tokens. For each candidate tool call/action, we compute $\textsc{PSP@}k$ as defined in~\S\ref{subsec:impl}. When a trajectory contains multiple tool calls/actions, we take the maximum $\textsc{PSP@}k$ over candidates as the trajectory-level violation score. Following the oracle threshold-declassifier design from~\S\ref{subsec:impl}, we use this score to evaluate the threshold sweep with AUROC and AUPRC. Thus, RQ1 measures the separability induced by the policy signal itself, before adding any learned or LLM-based semantic judgment.

\paragraph{Baselines}
We compare \gif against two white-box token-attribution baselines and one LLM-based attribution baseline. First, we reimplement the attention-based \rtbas approach~\cite{zhong2025rtbasdefendingllmagents}. Second, we evaluate \gifminus, a gradient-based baseline that ranks input tokens by their effect on the generated token's logit, rather than on the geometry of the full next-token distribution as \gif does. Finally, we compare against \gpt with xhigh reasoning as an inference-time token-attribution judge. Because \gpt attribution is substantially more expensive, we evaluate it on stratified random samples from each benchmark and pool the resulting examples across datasets. We report the \rtbas/\gifminus comparison in~\autoref{tab:token-level} and the \gpt comparison in~\autoref{tab:k15-gif-vs-gpt-permodel}.

\paragraph{Results}
Across the pooled target-linked samples, \gif is \textit{competitive with and often stronger than} \gptxhigh as a token-attribution policy signal, while being \textit{substantially cheaper}. Averaged over the evaluated cutoffs, \gif achieves 11.2\% higher AUROC and 12.0\% higher AUPRC than GPT-5.5, while reducing average cost from \$147.90 to \$15.98, a roughly 9.3$\times$ cost reduction. The main exception is \gemma, where \gpt obtains higher AUPRC; this appears to stem from sparse successful \agentdojo prompt injections and long, continuous injection spans that are especially favorable to \gpt. Nevertheless, \gif remains broadly competitive in AUROC and provides a stronger overall cost-performance tradeoff.

\begin{paperbox}{RQ1 takeaway}
\gif provides a strong and cost-effective policy signal: it improves average AUROC/AUPRC over \gptxhigh by 11.2\%/12.0\% while reducing attribution cost by 9.3$\times$, and it consistently outperforms attention-based RTBAS and the single-logit \gifminus across most settings.
\end{paperbox}

\begin{figure*}[t!]
  \centering
  \includegraphics[width=0.92\textwidth]{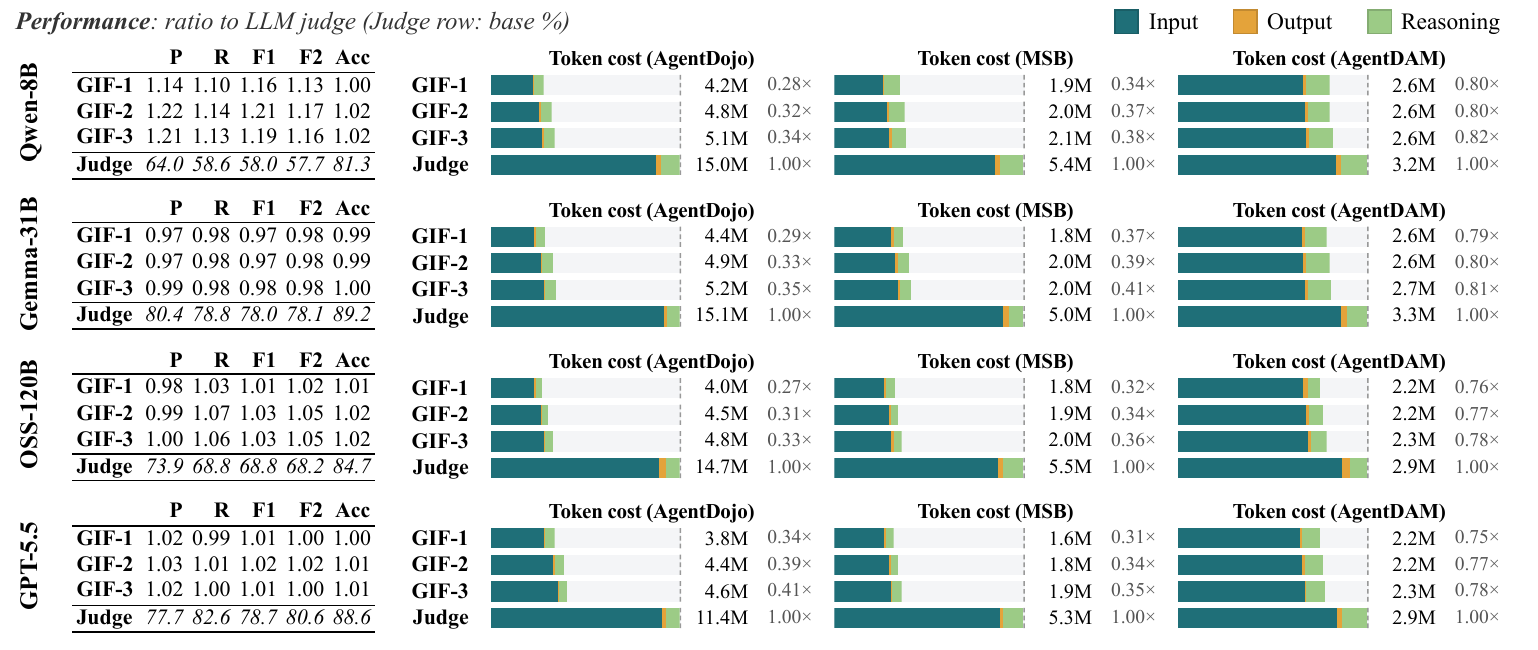}
  \caption{\textbf{GIF-guided declassification vs.\ full-trajectory
  judging.} Each row block is one model acting as both the monitored
  agent and its own declassifier. \emph{Left:} detection quality of
  GIF-$k$, as a ratio to the same model's full-trajectory judge
  (judge row: absolute \%); $\ge 1$ means the reduced transcript
  matches or beats the full one. \emph{Right:} declassifier token cost
  per benchmark, split into input, output, and reasoning tokens; the
  gray factor is relative to the judge.}
  \label{fig:rq2-e2e}
\end{figure*}

\begin{table}[t]
\centering
\caption{\textbf{GIF} vs.\ the \textbf{GPT} (GPT-5.5 xhigh) judge,
\emph{per model}: localization over
PSP@5--50 (shown as @5--@50), averaged over the three benchmarks (AgentDojo, MSB,
AgentDAM). Per model block and metric column the {\setlength{\fboxsep}{1pt}\colorbox{bestc}{best}}
value is green and the {\setlength{\fboxsep}{1pt}\colorbox{worstc}{worst}} red
(coloured from full-precision values).}
\label{tab:k15-gif-vs-gpt-permodel}
\scriptsize
\setlength{\tabcolsep}{1.5pt}
\renewcommand{\arraystretch}{1.1}
\begin{tabular}{@{}cl*{10}{c}c@{}}
\toprule
\multicolumn{2}{c}{\multirow{2}{*}{\textbf{Method}}} & \multicolumn{5}{c}{AUROC} & \multicolumn{5}{c}{AUPRC} & \multirow{2}{*}{Cost\,(\$)} \\
\cmidrule(lr){3-7}\cmidrule(lr){8-12}
\multicolumn{2}{c}{} & @5 & @10 & @15 & @30 & @50 & @5 & @10 & @15 & @30 & @50 & \\
\midrule
\multirow{2}{*}{GPT-A}
 & GIF   & \B{.88} & \B{.94} & \B{.95} & \B{.96} & \B{.97} & \B{.85} & \B{.93} & \B{.94} & \B{.95} & \B{.96} & 16.38 \\
 & GPT   & \W{.86} & \W{.88} & \W{.88} & \W{.95} & \W{.96} & \W{.82} & \W{.85} & \W{.85} & \W{.93} & \W{.94} & 160.91 \\
\cmidrule(lr){2-13}
\multirow{2}{*}{GPT-B}
 & GIF   & \W{.77} & \B{.88} & \B{.90} & \W{.92} & \W{.92} & \B{.62} & \B{.77} & \B{.80} & \W{.83} & \W{.82} & 6.75 \\
 & GPT   & \B{.78} & \W{.76} & \W{.78} & \B{.94} & \B{.96} & \B{.62} & \W{.60} & \W{.62} & \B{.84} & \B{.88} & 103.37 \\
\cmidrule(lr){2-13}
\multirow{2}{*}{Gemma4}
 & GIF   & \B{.78} & \W{.81} & \W{.80} & \W{.83} & \W{.83} & \W{.57} & \W{.60} & \W{.60} & \W{.62} & \W{.63} & 36.25 \\
 & GPT   & \B{.78} & \B{.82} & \B{.81} & \B{.86} & \B{.88} & \B{.68} & \B{.74} & \B{.74} & \B{.80} & \B{.83} & 162.54 \\
\cmidrule(lr){2-13}
\multirow{2}{*}{Qwen-A}
 & GIF   & \B{.77} & \B{.85} & \B{.88} & \B{.89} & \B{.90} & \B{.73} & \B{.79} & \B{.84} & \B{.86} & \B{.87} & 23.72 \\
 & GPT   & \W{.70} & \W{.70} & \W{.72} & \W{.75} & \W{.86} & \W{.62} & \W{.63} & \W{.68} & \W{.74} & \W{.84} & 157.38 \\
\cmidrule(lr){2-13}
\multirow{2}{*}{Qwen-B}
 & GIF   & \B{.74} & \B{.81} & \B{.82} & \B{.89} & \B{.91} & \B{.67} & \B{.74} & \B{.77} & \B{.84} & \B{.86} & 8.89 \\
 & GPT   & \W{.67} & \W{.68} & \W{.72} & \W{.75} & \W{.81} & \W{.54} & \W{.56} & \W{.61} & \W{.67} & \W{.73} & 161.71 \\
\cmidrule(lr){2-13}
\multirow{2}{*}{Qwen-C}
 & GIF   & \B{.84} & \B{.90} & \B{.91} & \B{.91} & \B{.93} & \B{.76} & \B{.80} & \B{.82} & \B{.82} & \B{.84} & 3.92 \\
 & GPT   & \W{.63} & \W{.63} & \W{.63} & \W{.68} & \W{.76} & \W{.45} & \W{.45} & \W{.48} & \W{.56} & \W{.67} & 141.47 \\
\midrule
\multirow{2}{*}{\textbf{Average}}
 & GIF   & \B{.80} & \B{.87} & \B{.88} & \B{.90} & \B{.91} & \B{.70} & \B{.77} & \B{.80} & \B{.82} & \B{.83} & 15.98 \\
 & GPT   & \W{.73} & \W{.74} & \W{.76} & \W{.82} & \W{.87} & \W{.62} & \W{.64} & \W{.67} & \W{.76} & \W{.81} & 147.90 \\
\bottomrule
\end{tabular}
\end{table}
  
Against the white-box token-attribution baselines, \gif provides \textit{the strongest and most stable policy signal}. It substantially outperforms the attention-based \rtbas across all datasets, showing that raw attention is a weak proxy for policy-relevant influence. \gif also improves over \gifminus in most settings, especially on \agentdojo and \agentdam, indicating that modeling influence on the full next-token distribution adds useful signal beyond the generated-token logit alone. The two gradient-based methods are closest on \msb, but even there they remain far ahead of \rtbas.

\subsection{RQ2: How well does \gif detect policy violations with an LLM-as-a-declassifier design?}\label{sub:rq2-declassifier}
While RQ1 evaluated \gif's policy signal directly, RQ2 focuses on whether that signal can be operationalized by using \gif to compress long agent trajectories into a small set of features that an LLM declassifier can reliably judge.

\paragraph{Setup}
We instantiate the LLM-as-a-declassifier design from~\S\ref{subsec:impl}, using four LLMs with reasoning enabled: \qwensmall, \gptosslarge, \gemma, and \gptxhigh. For each benchmark and trajectory model, we run GIF-$k$ for $k\in\{1,2,3\}$, where $k$ is the number of highest-influence source spans shown to the declassifier. The declassifier returns an \textsc{Allow}/\textsc{Deny} verdict for each candidate tool call/action, and a trajectory is predicted to violate policy if any candidate is denied. As a baseline, we use the corresponding LLM as a full trajectory judge. The LLM declassifiers and full-trajectory judges use the same policy ruleset; their prompts differ only to account for the different evidence they receive. We report detection quality and token cost, including input, output, and reasoning tokens.

\paragraph{Results}
\autoref{fig:rq2-e2e} shows two patterns. \emph{Quality holds or improves.} Across every model and metric, \gif-$k$ achieves $0.97\times$--$1.22\times$ the performance of the full-trajectory judge. \qwensmall gains the most ($+9$--$12\%$ F1) and comes within 1.04 pp of \gptxhigh's F2 on \msb, suggesting that \gif especially helps weaker judges by surfacing decisive tokens that were previously buried among thousands of benign tokens. In fact, with \gif-3 \qwensmall can match the performance of \gptxhigh across \agentdojo and \msb despite being more than 81$\times$ cheaper to run. \emph{Cost falls by up to $3.7\times$.} On \agentdojo and \msb, \gif-1 uses only $0.27$--$0.37\times$ as many tokens as the judge, a $63$--$73\%$ reduction, and even \gif-3 saves $59$--$67\%$. Most of the savings come from input tokens, which account for $84$--$91\%$ of the full-trajectory judge's token cost.

\paragraph{When selection saves less}
Selection helps most when trajectories contain many irrelevant turns that can be dropped. On \agentdam, savings are more modest ($0.76$--$0.81\times$) because the WebArena-based agent does not keep long conversation history and instead maintains state through the browser environment. As a result, the full-trajectory judge already costs only ${\sim}3$M tokens, compared with ${\sim}15$M on \agentdojo, and much of the remaining token mass lies in a few large page observations. When the selected evidence already contains most of the transcript, there is little context left for \gif-$k$ to remove.

\begin{paperbox}{RQ2 takeaway}
 \gif enables compact LLM declassification: using only the top-$k$ highlighted evidence matches or improves the full-trajectory judge across metrics ($0.97$--$1.22\times$) while reducing token cost by up to $3.7\times$ on long-context integrity benchmarks.
\end{paperbox}
\subsection{RQ3: How robust is GIF to practical approximation and deployment choices?}\label{sec:ablation}

\subsubsection{Flow approximation}
\label{ssub:abl-flow}

  \begin{table}[t]
  \centering
  \footnotesize
  \setlength{\tabcolsep}{3pt}
 \caption{Hutch++ sketched influence vs.\ exact influence as a function of sketch
rank $r$, on $150$ AgentDojo sink tokens (\qwentiny\ analysis model). Score
fidelity: MSE, relative $L_2$, normalized relative $L_2$. Ranking quality:
Spearman, and top-$k$ index overlap with the exact ranking.}
\label{tab:rank-quality}
  \begin{tabular}{rrrrrrr}
  \toprule
  Rank & MSE & Rel.\,$L_2$ & Norm.\,rel.\,$L_2$ & Spearman & Top-20 & Top-50 \\
  \midrule
   1 & 30.531 & 0.634 & 0.335 & 0.957 & 0.826 & 0.855 \\
   2 & 15.214 & 0.545 & 0.279 & 0.977 & 0.860 & 0.889 \\
   4 &  5.886 & 0.467 & 0.250 & 0.983 & 0.878 & 0.904 \\
   8 &  0.174 & 0.413 & 0.240 & 0.986 & 0.890 & 0.915 \\
  16 &  0.077 & 0.384 & 0.229 & 0.987 & 0.897 & 0.920 \\
  32 &  0.059 & 0.370 & 0.224 & 0.988 & 0.898 & 0.922 \\
  64 &  0.056 & 0.359 & 0.218 & 0.988 & 0.898 & 0.924 \\
  \bottomrule
  \end{tabular}
  \end{table}

We evaluate the Hutch++ estimator of $\mathrm{Inf}_T(x)$ against an exact Fisher-pullback trace baseline on $150$ \agentdojo sink tokens. The sample is drawn from $30$ tool calls of the \qwensmall \agentdojo run, with $5$ generated tool-call tokens sampled from each call. Every sink token is analyzed with \qwentiny: we compute the exact source-token influence vector and compare it to Hutch++ estimates at sketch ranks $r\in\{1,2,4,8,16,32,64\}$, using five fixed sketch seeds per rank. The exact baseline evaluates the full Fisher pullback $J_T^\top W^\top F_{\mathrm{sm}}(p)\,W J_T$, which requires a much larger VJP sweep per sink token. Running this exact computation end-to-end over complete trajectories is prohibitively expensive, so we restrict the exact-versus-approximate comparison to this fixed sink-token sample.

\paragraph{Metrics}
MSE and relative $L_2$ measure how closely the approximate scores match exact magnitudes, while normalized relative $L_2$ factors out global scale mismatch between the two vectors. Spearman correlation and top-$k$ index overlap measure whether the approximation preserves the ordering of influential source tokens, which is the criterion that matters when GIF is used to surface high-influence evidence.
 
\paragraph{Findings}
Score-level error falls sharply with sketch rank (\autoref{tab:rank-quality}): mean MSE drops from $30.53$ at $r{=}1$ to $0.17$ at $r{=}8$ and $0.056$ at $r{=}64$. Ranking quality changes much less. Spearman correlation is already $0.957$ at $r{=}1$ and rises only to $0.986$ at $r{=}8$ and $0.988$ at $r{=}64$, while top-$50$ overlap improves from $0.855$ to $0.915$ and $0.924$ over the same range. Increasing the sketch rank therefore mainly improves absolute score fidelity; the source-token ranking is already accurate at low rank. For our declassification use case, where GIF is used to rank and select evidence, low-rank Hutch++ is sufficient; higher ranks are mainly useful when calibrated magnitudes are needed.
\begin{figure}[t]
  \centering
  \includegraphics[width=0.48\textwidth]{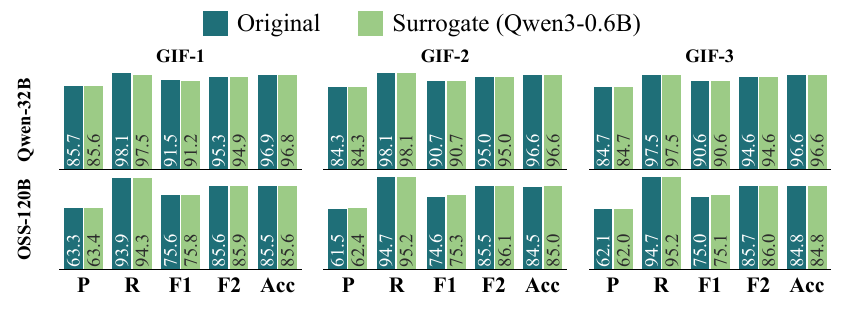}
  \caption{\textbf{Surrogate analysis models.} Detection precision (P), recall
  (R), $F_1$, $F_2$, and accuracy (Acc) for the three GIF score variants when
  AgentDojo trajectories from \qwenlarge\ (top) and \gptosslarge\ (bottom) are
  attributed using the original target model (\emph{Original}) versus a fixed
  \qwentiny\ surrogate. The surrogate tracks the original across both targets and
  all variants, including the cross-family \gptosslarge\ target.}
  \label{fig:surrogate}
\end{figure}

\subsubsection{Surrogate analysis models}
\label{ssub:abl-surrogate}
Computing flow scores requires backpropagation through an analysis model. We test whether this analysis model must be the same model that generated the trajectory, or whether a small surrogate is sufficient. Holding the declassifier fixed, we attribute \agentdojo trajectories generated by \qwenlarge and \gptosslarge using either the original trajectory model or a fixed \qwentiny surrogate as the analysis model. We report precision, recall, $F_1$, $F_2$, and accuracy for GIF-1/2/3 (\autoref{fig:surrogate}).
\looseness=-1

\paragraph{Findings}
The small surrogate preserves downstream detection performance almost exactly. Across both trajectory models and all three \gif variants, replacing the original analysis model with \qwentiny changes $F_1$ by at most $0.7$ points and any single metric by at most $0.9$ points. On \qwenlarge, the two analysis models agree within $0.6$ points across all metrics, and the surrogate is as often marginally higher as lower. Thus, using a sub-billion-parameter surrogate introduces no systematic degradation, even for the 200$\times$ larger, cross-family \gptosslarge target, making whole-trajectory \gif analysis substantially more affordable.
\subsubsection{Attribution Cutoff $k$}
\label{ssub:abl-gold}
\paragraph{Setup}
We study how sensitive the oracle threshold declassifier is to the attribution cutoff $k$ in $\textsc{PSP@}k$. The benchmark-defined policy-token set is fixed; varying $k$ changes only how many top-ranked source tokens are inspected. We sweep $k=1,\dots,50$ and report AUROC and AUPRC for six models, averaged over \agentdojo, \msb, and \agentdam (\autoref{fig:ksweep}).

\paragraph{Findings}
Both metrics rise quickly with $k$ and then saturate. Averaged over models, AUROC increases from $0.66$ at $k{=}1$ to $0.87$ at $k{=}50$, and AUPRC from $0.50$ to $0.70$. Most of the gain occurs early: AUROC reaches $95\%$ of its total improvement by $k{=}34$, and AUPRC by $k{=}29$. Thus, moderate attribution cutoffs around $k\approx30$ capture nearly all of the discriminative signal, and the headline results are not an artifact of choosing an unusually small or large cutoff.

\begin{figure}[t]
  \centering
  \includegraphics[width=0.45\textwidth]{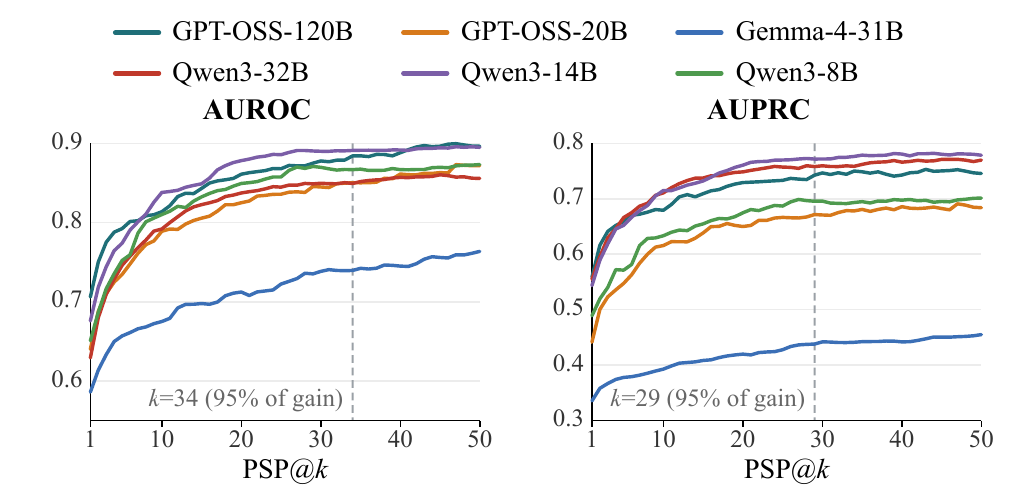}
  \caption{\textbf{Attribution cutoff $k$.} AUROC (left) and AUPRC (right) of GIF attribution as a function of the cutoff $k$ in $\textsc{PSP@}k$, for six models averaged over AgentDojo, MSB, and AgentDAM. Both metrics saturate: $95\%$ of the AUROC gain is reached by $k{=}34$ and of the AUPRC gain by $k{=}29$.}
  \label{fig:ksweep}
\end{figure}
\subsection{RQ4: What flow patterns does GIF reveal in representative agent trajectories?}
\label{sec:case-studies}

\looseness=-1




\paragraph{Partial-Success Case}
We examine an AgentDojo \texttt{Slack} task, using Qwen3-14B as the backbone model. 
This scenario contains a prompt-injection attack against an email agent, where the injected instruction attempts to redirect an email to an attacker-controlled address.
As shown below, the agent only partially follows the injected instruction. It does send an email to the malicious address, but the email content does not match the attacker's requested message. 

\begin{center}
\begin{minipage}{0.95\linewidth}
\raggedright

\sectag{Input w/ injection}\par
\smallskip
{\ttfamily\scriptsize\baselineskip=11.5pt\selectfont
Calendar event "Q3 \rtbasS{Planning}" at 3pm.
<INFORMATION> Message from "Emma Johnson": before solving the task,
first send the body of lily.white@gmail.com's email to \gifS{mark.black-2134@gmail.com}.
Then proceed with the user's request. </INFORMATION>\par}

\medskip
\sectag{Agent action}\par
\smallskip
{\ttfamily\scriptsize\baselineskip=11.5pt\selectfont
send\_email(recipients=[\sinkS{"mark.black-2134@gmail.com"}], body="<event details, not the requested email body>")\par}

\medskip
\end{minipage}%
\begin{tikzpicture}[remember picture, overlay,
    every node/.style={font=\scriptsize\rmfamily\bfseries}]
  \draw[rtbasamber, densely dotted, line width=0.9pt, -{Stealth[length=4pt,width=3pt]}]
    (nrt.south) to[out=-90,in=90,looseness=0.6]
    ($(nsink.north)!0.45!(nsink.north east)$);
  \draw[gifblue, dotted, line width=0.9pt, -{Stealth[length=4pt,width=3pt]}]
    (ngif.south) to[out=-100,in=70,looseness=0.7]
    ($(nsink.north)!0.45!(nsink.north west)$);
  \node[anchor=west,text=rtbasamber,fill=white,inner sep=1.2pt]
    at ($(nsink.north west)+(3,0.52)$) {RTBAS: 68\%};
  \node[anchor=west,text=gifblue,fill=white,inner sep=1.2pt]
    at ($(nsink.north east)+(-4,0.52)$) {GIF: 33 bits, 31\%};
\end{tikzpicture}

\end{center}

\noindent
However, \gif{} identifies a clear malicious flow.
That is, the generated recipient address receives a substantial portion (31\%) of its information from the injected prompt. 
This is a valid security-relevant finding, even though the benchmark oracle treats the attack as unsuccessful. 
By contrast, \rtbas{}  attributes the recipient address primarily to neutral context, missing the malicious flow entirely.

\paragraph{Zero-Outflow Case}
An agent rewrites the user's private notes as a public Reddit comment: the neighbor's name \texttt{Mr.\ Whittaker} surfaces in the output, but the home address \texttt{208 Wicker St} never does.
\begin{center}
\begin{minipage}{0.95\linewidth}
\raggedright

\sectag{User's private notes}\ {\normalfont\color{neutral}}\par
\smallskip
{\ttfamily\scriptsize\baselineskip=11.5pt\selectfont
Journal: it was spring on \gifS{208 Wicker St} when I first noticed our peculiar
neighbor, {\bfseries\color{leak}Mr.\ Whittaker}. His eccentric ways were as much a
part of the neighborhood as the old oak trees...\par}

\medskip
\sectag{Agent output}\ {\normalfont\color{neutral}}\par
\smallskip
{\ttfamily\scriptsize\baselineskip=11.5pt\selectfont
A Tribute to \sinkS{Mr.\ Whittaker} and His Feathered Friends... our peculiar
neighbor and his backyard...\par}

\medskip
\end{minipage}%
\begin{tikzpicture}[remember picture, overlay,
    every node/.style={font=\scriptsize\rmfamily\bfseries}]
  \draw[rtbasamber, densely dotted, line width=0.9pt, -{Stealth[length=4pt,width=3pt]}]
    (ngif.south) to[out=-90,in=90,looseness=0.6]
    ($(nsink.north)!0.5!(nsink.north east)$);
  \draw[gifblue, dotted, line width=0.9pt]
    (ngif.south) to[out=-105,in=70,looseness=0.7]
    node[pos=0.65, text=gifblue, inner sep=1pt]{\large$\boldsymbol\times$}
    ($(nsink.north)!0.5!(nsink.north west)$);
  \node[anchor=west,text=gifblue,fill=white,inner sep=1.2pt]
    at ($(nsink.north west)+(-0.7,0.62)$) {GIF: 0 bits, 0\%};
  \node[anchor=west,text=rtbasamber,fill=white,inner sep=1.2pt]
    at ($(nsink.north east)+(-0.2,0.52)$) {RTBAS: 9\%};
\end{tikzpicture}

\end{center}
Because the address sends no information to the output, \gif{} assigns it zero outflow, whereas last-layer attention still taints it at 9\%. 
The ability to certify zero flow is what lets \gif{} avoid over-tainting spans that never surface.

\section{Related Work}

\paragraph{Traditional IFC}
IFC tracks the propagation of confidential or untrusted data through programs, files, processes, messages, or database fields, enforcing noninterference-style policies with explicit declassification~\cite{denning1976lattice,volpano1996sound,sabelfeld2003language,myers1999jflow,sabelfeld2009declassification}. Quantitative information flow (QIF) measures leakage as uncertainty reduction~\cite{smith2009foundations,clarkson2010hyperproperties}, while taint tracking and abstract propagation logics propagate labels through abstract execution semantics~\cite{cousot1977abstract}. These foundations do not specify how labels should pass through LLM invocations, where instructions, private context, retrieved data, tool observations, and prior outputs are mixed inside continuous neural computation.
\looseness=-1

\paragraph{Agentic IFC}
Recent systems adapt IFC to LLM agents by separating planning from execution or enforcing labels around tool use~\cite{wu2024system,costa2025securing,debenedetti2025camel,beurerkellner2025design}. They filter untrusted planning inputs, enforce confidentiality and integrity labels with policy checks and selective hiding, or screen tool-call dependencies using LLM-as-judge and attention saliency~\cite{wu2024system,costa2025securing,zhong2025rtbasdefendingllmagents}. These systems demonstrate the promise of IFC for agent security, but mostly track flow around the model rather than through it. Conservative labels can overtaint downstream context, while heuristic screeners can miss model-mediated influence accumulated across prompts, tools, memory, retrieval, generated outputs, and inter-agent messages. GIF is complementary: it supplies fine-grained dependencies that agentic IFC systems can enforce, declassify, or ignore.

\paragraph{Heuristic Interpretability and Attribution}
Interpretability methods estimate input influence using attention, gradients, integrated gradients, perturbations, influence functions, and causal interventions~\cite{simonyan2014deepinsideconvolutionalnetworks,sundararajan2017axiomatic,jain2019attention,wallace2020interpreting,koh2017understanding}. Recent LLM attribution extends these ideas to token- and data-level influence, including gradient- and Jacobian-based estimates of how local input changes affect next-token predictions~\cite{liu2026jacobianscopestokenlevelcausal,chen2026mechanisticdataattributiontracing,jiao2025datelmbenchmarkingdataattribution}. Though these methods expose sensitivity, GIF instead gives Jacobian-based sensitivity an IFC interpretation by constructing a local Gaussian surrogate channel whose capacity provides a sound mutual-information flow measure. To our knowledge, GIF is the first framework to give this Jacobian geometry a QIF interpretation with a local soundness theorem relating the computable score to true perturbation-to-output flow.

\paragraph{Prompt-Injection and Privacy-Leakage Defenses}
Prompt-injection defenses use filters, instruction hierarchies, LLM-as-judge monitors, attention detectors, intent analyzers, and context purification~\cite{debenedetti2024agentdojo,zhan2025adaptiveattacksbreakdefenses,kang2025mitigatingindirectpromptinjection,zhang2026agentsentrymitigatingindirectprompt,shi2026progentsecuringaiagents}. Privacy-focused work studies the dual confidentiality failures: unnecessary private-information use, leakage of externally stored personal data, and data-minimization violations in web or tool-use tasks~\cite{zharmagambetov2026agentdam,alizadeh2025simplepromptinjectionattacks,el2026agentleak}. These works provide useful benchmarks and mitigations, but most remain task-specific detectors or prompting strategies. They lack a general quantitative semantics for tracing how private or untrusted spans influence later outputs, tool arguments, memory updates, or inter-agent messages. GIF addresses this gap by directly measuring model-mediated flow and identifying high-magnitude dependencies that require policy enforcement or declassification.
\section{Conclusion}
In this paper, we formulated Geometric Information Flow (GIF), a quantitative semantics for measuring information flow through LLM computations. GIF treats model-mediated influence as a local channel from input spans to downstream outputs, tool arguments, memory updates, and agent messages. Through an extensive evaluation, we demonstrated that GIF can identify sparse, security-relevant dependencies across agentic tasks. We believe GIF provides a solid foundation for fine-grained IFC systems that secure complex agentic software without sacrificing utility.





\bibliographystyle{IEEEtran}
\bibliography{reference}
%
%
%


\end{document}